\documentclass[runningheads]{llncs}
\usepackage[T1]{fontenc}
\usepackage{graphicx}
\usepackage{booktabs}
\usepackage[misc]{ifsym}

\usepackage{mwe}

\usepackage[colorlinks=true, linkcolor=blue, citecolor=blue, urlcolor=blue]{hyperref}

\usepackage{amsmath,amsfonts,amssymb} 
\usepackage{bm} 
\usepackage{algorithm}
\usepackage{algpseudocode}
\usepackage{graphicx} 
\usepackage{subcaption} 
\usepackage{textcomp} 
\usepackage{stfloats} 
\usepackage{url} 
\usepackage{verbatim} 
\usepackage{color} 
\usepackage{booktabs} 
\usepackage{multirow} 
\usepackage{colortbl}

\usepackage[table]{xcolor}
\definecolor{grey}{rgb}{0.1,0.1,0.1}
\usepackage{amsthm}

\author{
Lijie Hu\inst{*,1,2},
Songning Lai\inst{*,1,2},
Yuan Hua\inst{*,1,2,3}, \\
Shu Yang\inst{1,2},
Jingfeng Zhang\inst{1,2,4},
Di Wang\inst{\dagger,1,2}
}

\authorrunning{L.Hu et al.}

\institute{Provable Responsible AI and Data Analytics (PRADA) Lab 
\and
King Abdullah University of Science and Technology
\and
Tsinghua University 
\and
University of Auckland}

\begin{document}

\title{Stable Vision Concept Transformers for Medical Diagnosis}

\titlerunning{Stable Vision Concept Transformers}

\maketitle              

\def\thefootnote{*}\footnotetext{Equal Contribution.}
\def\thefootnote{$\dagger$}\footnotetext{Corresponding Author.}

\begin{abstract}
Transparency is a paramount concern in the medical field, prompting researchers to delve into the realm of explainable AI (XAI). Among these XAI methods, Concept Bottleneck Models (CBMs) aim to restrict the model's latent space to human-understandable high-level concepts by generating a conceptual layer for extracting conceptual features, which has drawn much attention recently. However, existing methods rely solely on concept features to determine the model’s predictions, which overlook the intrinsic feature embeddings within medical images. To address this utility gap between the original models and concept-based models, we propose \underline{\textbf{V}}ision \underline{\textbf{C}}oncept \underline{\textbf{T}}ransformer (\underline{\textbf{VCT}}). Furthermore, despite their benefits, CBMs have been found to negatively impact model performance and fail to provide stable explanations when faced with input perturbations, which limits their application in the medical field. To address this faithfulness issue, this paper further proposes the \underline{\textbf{S}}table \underline{\textbf{V}}ision \underline{\textbf{C}}oncept \underline{\textbf{T}}ransformer (\underline{\textbf{SVCT}}) based on VCT, which leverages the vision transformer (ViT) as its backbone and incorporates a conceptual layer. SVCT employs conceptual features to enhance decision-making capabilities by fusing them with image features and ensures model faithfulness through the integration of Denoised Diffusion Smoothing. Comprehensive experiments on four medical datasets demonstrate that our VCT and SVCT maintain accuracy while remaining interpretability compared to baselines. Furthermore, even when subjected to perturbations, our SVCT model consistently provides faithful explanations, thus meeting the needs of the medical field.

\end{abstract}
\section{Introduction}
As the field of medical image analysis continues to evolve, deep learning models and methods have demonstrated excellent performance in tasks such as image recognition and disease diagnosis  \cite{Kermany2018IdentifyingMD}. However, these advanced deep learning models are usually regarded as black boxes and lack credibility and transparency. Especially in the medical field, this opacity makes it difficult for physicians and clinical professionals to trust the predictions of the models. Thus, the requirement for interpretability of model decisions is more urgent in the medical field \cite{hu2024towards,huai2020global,huai2020towards}.

The healthcare field, characterized by stringent requirements for trustworthiness, necessitates models that not only exhibit high performance but are also comprehensible and can be trusted by practitioners. Therefore, Explainable Artificial Intelligence (XAI) has become one of the hotspots for research and development. By introducing interpretability, XAI tries to make the decision-making process of deep learning models more transparent and understandable. Some compelling interpretable methods, such as attention mechanisms \cite{vaswani2017attention,hu2025towards}, saliency maps \cite{7780688}, DeepLIFT and Shapley values \cite{lundberg2017unified}, and influence functions \cite{koh2020understanding,hu2024dissecting}, attempt to provide users with visual explanations about model decisions. However, while these post-hoc explanatory methods can provide useful information, there is still a certain disconnect between their explanations and model decisions, and these explanations are generated after model training and fail to participate in the model learning process. Some studies \cite{rudin2019stop,lai2023faithful,hu2024faithful} have shown that post-hoc is sensitive to slight changes in the input, making the post-hoc methods misleading as they could provide explanations that do not accurately reflect the model’s decision-making process.

\begin{figure*}[!th]
    \centering
    \vspace{-15pt}
    \includegraphics[width=1\linewidth]{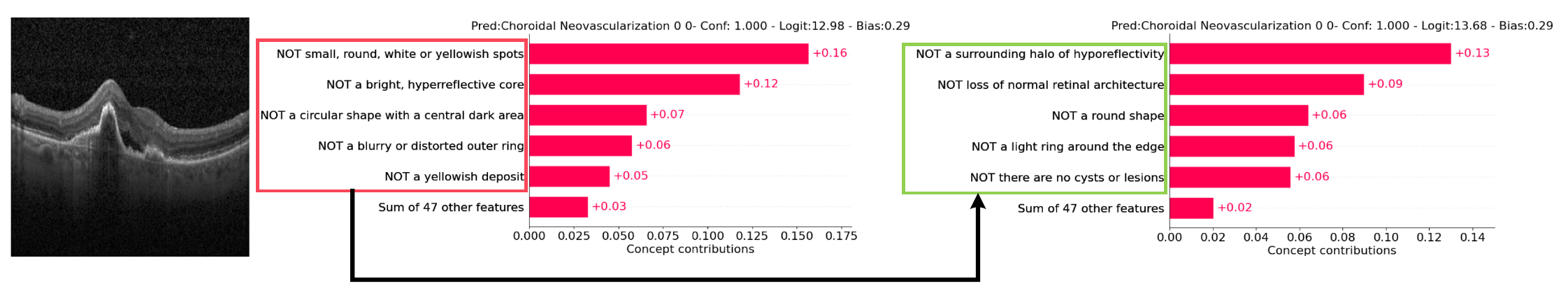}
    \caption{An example of VCT framework on OCT2017 dataset \cite{Kermany2018IdentifyingMD}. The leftmost figure displays the input image, while the adjacent one on the left shows the concept output without perturbations. In contrast, the figure on the right presents the concept output after applying input perturbations, resulting in noticeable changes. }
    \label{fig:2}
\vspace{-15pt}
\end{figure*}

Therefore, researchers have shown interest in self-explained methods. Among them, concept-based methods have attracted a lot of attention. These approaches strive to incorporate interpretability into machine learning models by establishing connections between their predictions and concepts that are understandable to humans. As an illustration, the Concept Bottleneck Model (CBM) \cite{koh2020concept} initially forecasts an intermediate set of predefined concepts, subsequently utilizing these concepts to make predictions for the final output. \cite{oikarinen2023labelfree} introduce Label-free CBM, a novel framework designed to convert any neural network into an interpretable CBM without the need for labeled concept data compared to the original CBM. These inherently interpretable methods provide concept-based explanations, which are generally more comprehensible than post-hoc approaches. However, many existing methods rely solely on concept features to determine the model's predictions. These approaches overlook the intrinsic feature embeddings within medical images. For instance, \cite{sarkar2022framework} solely utilizes concept labels to supervise the concept prediction results of the entire image. This oversight can lead to a decrease in classification accuracy, which is suggested to stem from the inefficient utilization of valuable medical information. Therefore, a significant challenge in the field of medical imaging is how to maintain a high level of accuracy while incorporating interpretability.

To address the aforementioned challenges, we propose Vision Concept Transformer (VCT), a novel medical image processing framework that is interpretable and maintains high performance. Vision Transformers (ViTs) \cite{dosovitskiy2020image} have achieved state-of-the-art performance for various vision tasks, showing good robustness in prediction. Thus, in the VCT framework, we utilize ViTs as the foundational network. To enhance interpretability, we employ a label-free methodology for generating the conceptual layer. Moreover, unlike previous CBMs, which only use conceptual features for prediction, in the VCT framework, we integrate conceptual features with image features, utilizing the conceptual layer as supplementary information to augment decision-making. This integration effectively addresses the issue of accuracy degradation associated with a singular label-free CBM, ensuring interpretability without compromising accuracy. 

While VCT keeps the interpretability of CBMs, it also 
inherits their interpretability instability when facing perturbations or noise in the input. Specifically, adding slight noise to the input image can significantly change the top-$k$ important concepts given by CBMs (see Figure \ref{fig:2} for an example), i.e., the top $k$-indices of the concept vector. Instability is a common issue in deep learning interpretation methods, making it challenging to understand model reasoning \cite{hu2022seat}, especially with unlabeled data and self-supervised training \cite{ghorbani2018interpretation}. As in real medical scenarios, there is always natural and inherent noise or some adversarial examples manipulated by attackers \cite{apostolidis2021survey,fu2025short,xu2023llm}. Thus, VCT cannot be a faithful explainable tool for these applications. 

To address the faithfulness issue, by using the Denoised Diffusion Smoothing method, we can smoothly and directly transform VCT into a Stable Vision Concept Transformer (SVCT) framework that is capable of providing stable interpretations despite perturbations to the inputs, the structure is shown in Figure \ref{fig:e over}. Our contributions can be summarised as follows.

\begin{itemize}
  \item We proposed the VCT framework, transforming ViTs into an interpretable CBM. VCT integrates conceptual features with image features, utilizing conceptual features as auxiliary decision-making components. This effectively addresses the performance degradation issue in existing CBMs due to inefficient utilization of medical information.
  \item To further enhance the interpretability stability of VCT, we propose a formal mathematical definition of an SVCT, which ensures that the top-$k$ index of its conceptual vectors remains relatively stable under slight perturbations. We utilize a Denoised Diffusion Smoothing (DDS) method to obtain an SVCT. Moreover, we theoretically proved that our method satisfies the properties of SVCT. 
  \item We conducted extensive experiments on four medical datasets to validate the superiority of SVCT in the medical domain. First, we demonstrate that our SVCT is more accurate and interpretable than other CBM approaches. Secondly, we verified that the SVCT model still provides stable explanations under perturbations.
\end{itemize}
\section{Related Work}

\paragraph{Concept Bottleneck Models.}
Concept Bottleneck Model (CBM) \cite{koh2020concept} stands out as an innovative deep-learning approach applied to image classification and visual reasoning. It introduces a concept bottleneck layer into deep neural networks, enhancing model generalization and interpretability by learning specific concepts. However, CBM faces two primary challenges: its performance often lags behind that of original models lacking the concept bottleneck layer, attributed to incomplete information extraction from the original data to bottleneck features. Additionally, CBM relies on laborious dataset annotation \cite{ismail2023concept,hu2024semi,hu2024editable}. Researchers have explored solutions to these challenges. \cite{chauhan2023interactive} extend CBM into interactive prediction settings, introducing an interaction policy to determine which concepts to label, thereby improving final predictions. \cite{oikarinen2022label} address CBM limitations and propose a novel framework called Label-free CBM. This innovative approach enables the transformation of any neural network into an interpretable CBM without requiring labeled concept data, all while maintaining high accuracy \cite{yuksekgonul2023posthoc}. However, most of the existing CBMs use only conceptual features for prediction, which can cause a degradation in prediction performance and make them unsuitable for medical scenarios.

\paragraph{Faithfulness in Explainable Methods.}
Faithfulness is an important property that should be satisfied by explanatory models, which ensures that the explanation accurately reflects the true reasoning process of the model \cite{jacovi2020faithfully,hu2024hopfieldian,gou2023fundamental}. Stability is crucial to the faithfulness of the interpretation. Some preliminary work has been proposed to obtain stable interpretations. For example, \cite{yeh2019infidelity} theoretically analyzed the stability of post-hoc explanations and proposed the use of smoothing to improve the stability of explanations. They devised an iterative gradient descent algorithm for obtaining counterfactual explanations, which showed desirable stability. However, these techniques are designed for post-hoc explanations and cannot be directly applied to attention-based mechanisms like ViTs.

\paragraph{Interpretability in Medical Image Classification.}
In the research of interpretable artificial intelligence in medical image analysis, \cite{yan2023robust} proposes a new method to construct a robust and interpretable medical image classifier using natural language concepts, and it has been evaluated on multiple datasets. \cite{sarkar2022framework} focuses on self-explanatory deep models, introducing a model that implicitly learns conceptual explanations during training by adding an explanation generation module. These methods collectively enhance the interpretability of the model. However, the existing interpretability methods face two main issues. Firstly, they rely solely on concept features for decision-making, leading to insufficient utilization of valuable information in medical images and resulting in a performance decline in medical image processing. Secondly, existing methods exhibit instability when confronted with noise, failing to provide faithful explanations. Therefore,  our work aims to ensure good performance while maintaining interpretability and providing faithful explanations to address these issues. See Appendix F for more details.
\section{Stable Vision Concept Transformer}
In this section, we propose the Stable Vision Concept Transformer (SVCT) framework. Specifically, we first leverage the Label-free Concept Bottleneck Model \cite{oikarinen2023labelfree} to transform the ViT network into an interpretable CBM without concept labels,  which is an automated, scalable, and efficient fashion to address the core limitations of existing CBMs. We then fuse the concept features with the ViTs features as decision-aiding features, which not only improves the interpretability of the model but also ensures a high degree of accuracy. To obtain an SVCT, we adopt Denoised Diffusion Smoothing (DDS) to turn it into an SVCT. 

Our model consists of the following six steps, which are illustrated in Figure \ref{fig:e over} - \textbf{Step1:} The ViT model is trained on the target task, and VCT is transformed into SVCT by inserting the DDS method. \textbf{Step2:} We generate initial concept set based on the target task and filter out unwanted concepts using a series of filters. \textbf{Step3:} Compute embeddings by the backbone on the training dataset and obtain the concept matrix. \textbf{Step4:} Learn projection weights $W_c$ to create a Concept Bottleneck Layer (CBL). \textbf{Step5:} Fuse the concept features with the ViTs features. \textbf{Step6:} Learn the weights $W_F$ of the sparse final layer to make predictions. Detailed notations can be found in Table~\ref{tab:notions}. We first introduce VCT for convenience. 

\begin{figure*}[th]
    \centering
    \includegraphics[width=1\linewidth]{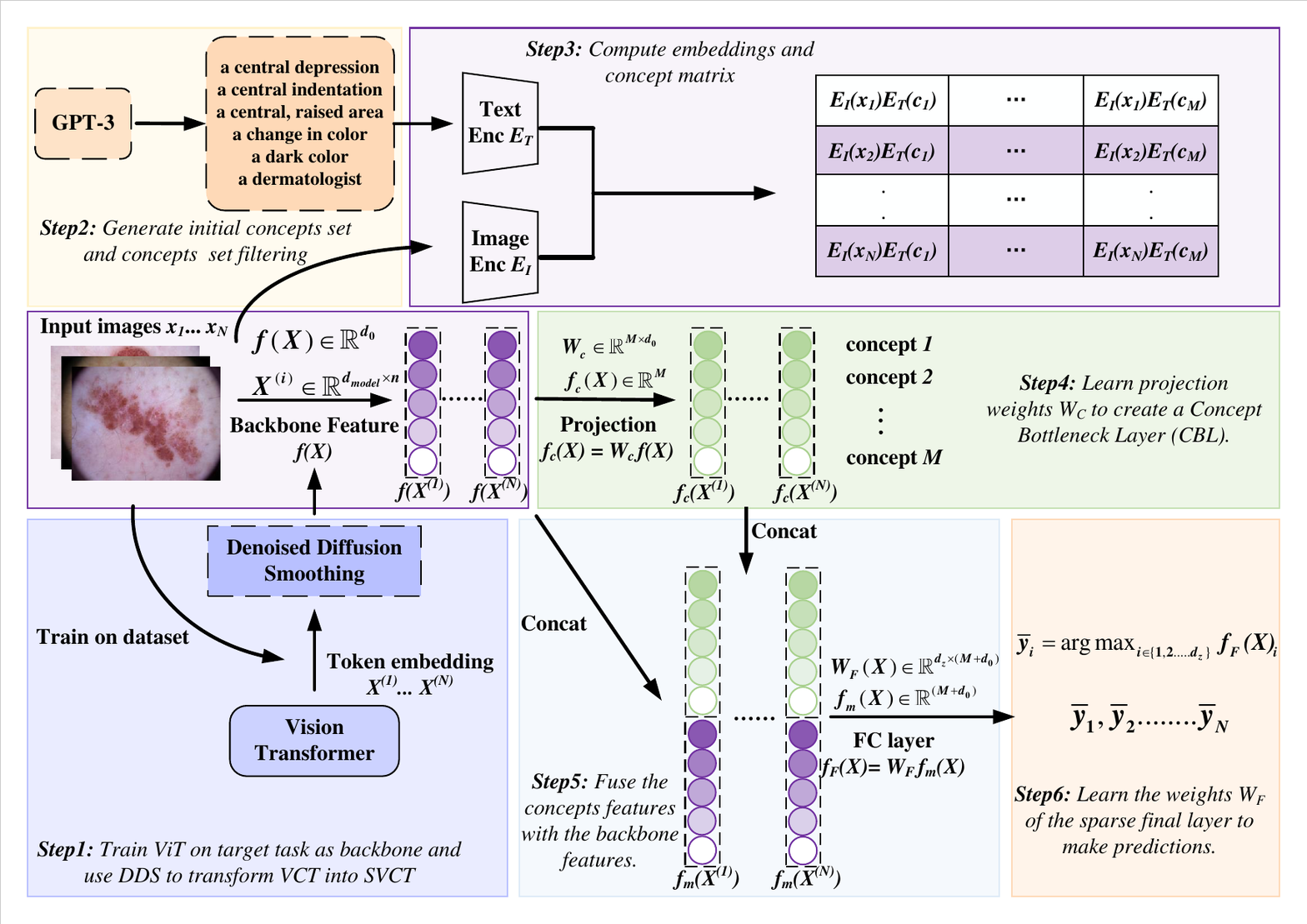}
    \caption{Overview of our Stable Vision Concept Transformer (SVCT) model.}
    \label{fig:e over}
\vspace{-15pt}
\end{figure*}

\subsection{Vision Concept Transformer}
In this section, we introduce the vision concept transformer. Before that, it is necessary to pre-train the ViT model $f$ on the target task dataset as a backbone for the VCT framework.

\noindent{\bf Label-free CBMs.}
We use the label-free CBM \cite{oikarinen2023labelfree} to get concept feature $f_c\left(X\right) \in \mathbb{R}^M$, where $M$ is the number of concepts. Firstly, we obtain a concept set and use it as human-understandable concepts in the concept bottleneck layer (See  Appendix D and E for details). Next, we need to learn how to project from the feature space $\mathbb{R}^{d_0}$ of the backbone network to an interpretable feature space $\in \mathbb{R}^M$ that corresponds to the set of interpretable concepts in the axial direction. We use a way of learning the projection weights $W_c \in \mathbb{R}^{M\times d_0}$ without any labeled concept data by utilizing CLIP-Dissect \cite{oikarinen2023clipdissect}. We can learn about a bottleneck conceptual layer and get the concept feature 
\begin{equation}\label{eq:100}
    f_c\left(X\right) = W_cf(X) \in \mathbb{R}^M.
\end{equation}

\noindent{\bf Concat ViT feature and concept feature.}
Now that we have learned about the conceptual bottleneck layer and get $W_c \in \mathbb{R}^{M\times d_0}$. In VCT, the conceptual features are no longer used as the only features for classification. According to previous studies, based on the conceptual features alone will degrade the accuracy of the model. Therefore, here we use the conceptual features as the supplementary features, which are fused with the features extracted from the backbone network, and this feature fusion makes the VCT able to ensure accuracy improvement while having a better explanatory nature. Specifically, we define $f_m(X)=\text{concat}(f(X), f_c(X))$, where $f_m\left(X^{(i)}\right) \in \mathbb{R}^{M+d_0}$,
and we define a feature of VCTs for prediction as follows:
\begin{equation}\label{eq:F1}
F(X) = \text{concat}(f(X),W_cf(X)).
\end{equation}

\noindent{\bf Final classification layer.}
The next goal is to learn the final predictor using the fully connected layer $W_F \in \mathbb{R}^{d_z\times \left(M+d_0\right)}$, where $d_z$ represents the final number of predicted categories. For each input $X$, we have access to its predictive distribution through the final classification layer.

\subsection{Stable VCT}
As we mentioned in the introduction and Figure~\ref{fig:2}, CBMs and VCT have an interpretation instability issue, i.e., a slight perturbation on the input could change the top-$k$ concepts in the concept vector (concept feature in VCT). Here we aim to address the instability issue. We first give the definition of the top-$k$ overlap ratio for two (concept) vectors, 

\begin{definition}
\label{def:1}
For vector $x\in \mathbb{R}^n$, we define the set of top-$k$ component $T_k(\cdot)$ as 
\begin{equation*}
    T_k(x)=\{i: i\in [d] \text{ and } \{|\{x_j\geq x_i: j\in [n]\}|\leq k\} \}.
\end{equation*}
For two vectors $x$, $x'$, their top-$k$ overlap ratio $V_k(x, x')$ is defined as $V_k(x, x')=\frac{1}{k} | T_k(x) \cap T_k(x')|$. 
\end{definition} 

\begin{definition}[Stable VCTs]
\label{def:2}
Giving $M$ number of concepts, a norm $\|\cdot\|$, and a divergence metric $D$, we call a function $g: \mathbb{R}^{d_{model} \times n} \rightarrow \mathbb{R}^{M}$ is an $(R, D, \gamma,  \beta, k,  \|\cdot\|)$-stable concept module for VCTs  if for any given input data $X$ and for all $X^{\prime} \in \mathbb{R}^{d_{model} \times n}$ such that $\|X-X^{\prime}\| \leq R$:
\begin{enumerate}
\item [(1)] (Explanation Stability) $V_k\left(g\left(X^{\prime}\right), g(X)\right) \geq \beta$.       
\item [(2)] (Prediction Robustness) $D\left(\bar{y}(X), \bar{y}\left(X^{\prime}\right)\right) \leq \gamma$, where $\bar{y}(X), \bar{y}\left(X^{\prime}\right)$ are the prediction distribution of VCTs based on $g(X), g\left(X^{\prime}\right)$ respectively. 
\end{enumerate}
We call the models of VCTs based on $g$ as SVCTs. 
\end{definition}
Intuitively, for input $X$, $g(X)$ is its concept vector. Thus, the first condition of SVCT ensures that the $k$-most important concepts will not change much, even if there are some perturbations on the input. The second one guarantees that the prediction of SVCT is also stable against perturbation, which inherits the good performance of VCT. For the parameters, $R$ represents the stable radius. Within this radius, $g$ is a stable concept module, $D$ is the Rényi divergence between two distributions (we denote it as $D_\alpha$). $\gamma$ is a similarity coefficient, and as $\gamma$ gets smaller, $g$ is more robust. $\beta$ is the stability coefficient, which measures the stability of the interpretation, and as $\beta$ gets larger, $g$ is more stable. In this paper, $\|\cdot\|$ is the $\ell_2$-norm (if we consider $X$ as a $d=d_{model} \times n$ dimensional vector). We can show if the prediction distribution is robust under Rényi divergence, then the prediction will be unchanged with perturbations on input (shown in Theorem \ref{thm:0})~\cite{zheng2020towards}.

\begin{theorem}\label{thm:0}
If a function is a $(R, D_\alpha, \gamma,  \beta, k,  \|\cdot\|)$-{stable concept module} for VCTs, then if 
\begin{equation*}
    \gamma \leq - \log(1 - p_{(1)} - p_{(2)} + 2(\frac{1}{2}(p_{(1)}^{1 - \alpha} + p_{(2)}^{1 - \alpha}))^{\frac{1}{1- \alpha}}),
\end{equation*}
we have for all $X'$ such that where $\|X - X'\| \leq R$,
\begin{equation*}
\arg\max_{h \in \mathcal{H}} \mathbb{P}(\bar{y}(X)=h) = \arg\max_{h \in \mathcal{H}} \mathbb{P}(\bar{y}(X')=h),    
\end{equation*}
where $\mathcal{H}$ is the set of classes,
$p_{(1)}$ and $p_{(2)}$ refer to the largest and the second largest probabilities in $\{ p_i \}$, where $p_i$ is the probability that $\bar{y}(X)$ returns the $i$-th class.
\end{theorem}

\paragraph{Finding Stable Vision Concept Transformers.} 
Motivated by \cite{ho2020denoising}, we propose a method called Denoised Diffusion Smoothing (DDS) to obtain SVCTs. The process is as follows: we use randomized smoothing to the VCT and then apply a denoised diffusion probabilistic model to the perturbed input. With this processing, we can transform a VCT into an SVCT, and its corresponding concept module becomes a stable concept module. Specifically, for a given input image $x$, its corresponding token embedding is $X$. We add some randomized Gaussian noise to $X$, i.e., $\tilde{X}=X+S$, where $S \sim \mathcal{N}\left(0, \sigma^2 I_{d_{model} \times n}\right)$. Then we will use some denoised diffusion models to denoise $\tilde{X}$ to get $\hat{X}$. We then take the obtained $\hat{X}$ as a new input to get concept feature $f_c(\hat{X})$ in (\ref{eq:100}) and go through the remaining structures of the VCT to get the final prediction. 

Specifically, for a given input $X$, randomized smoothing is done by augmenting the data points of an image by adding additive Gaussian noise to the image, which we can denote as $X_{\mathrm{rs}} \sim \mathcal{N}\left(X, \sigma^2 \mathbf{I}\right)$. Diffusion models rely on a particular form of noise modeling, denoted as $X_t \sim \mathcal{N}\left(\sqrt{\beta_t} X,\left(1-\beta_t\right) \mathrm{I}\right)$. Where $\beta_t$ is a constant related to time step $t$. Thus, if we want to use a diffusion model for randomized smoothing, we need to establish a link between the parameters of the two noise models. The DDS model used in this paper multiplies $X_{rs}$ by the factor $\sqrt{\beta_t}$, thus satisfying the requirement of the noise mean, and accordingly, in order to satisfy the requirement of the variance, we can obtain the equation $\sigma^2=\frac{1-\beta_t}{\beta_t}$. As the time step changes, $\sigma^2$ changes as $\beta_t$ changes because $\beta_t$ is a constant with respect to the time step. But it can be computed at every time step, and by using this, we are able to obtain $X_{t^*}=\sqrt{\beta_{t^*}}(X+S)$, where $ S \sim \mathcal{N}\left(0, \sigma^2 \mathbf{I}\right)$. Such a form of noise is consistent with the form on which the diffusion model depends, and we can use the diffusion model on $X_{t^*}$ to obtain denoised sample $\hat{X}=\operatorname{denoise}\left(X_{t^*}; t^{*}\right)$. In this paper, we repeat this process several times to improve robustness. 

In the following, we show that $\tilde{w}=f_c(\hat{X})$ is a stable concept feature satisfying Definition \ref{def:2} if $\sigma^2$ satisfies some condition. Before showing the results, we first provide some notations. For input image $x$, we denote $\tilde{w}_{i^*}$ as the $i$-th largest component in $\tilde{w}(x)$. Let $k_0=\lfloor (1-\beta)k \rfloor +1$ as the minimum number of changes on $\tilde{w}(x)$ to make it violet the  $\beta$-top-$k$ overlapping ratio with $\tilde{w}(x)$. Let $\mathcal{S}
$ denote the set of last $k_0$ components in top-$k$ indices and the top $k_0$ components out of top-$k$ indices. Then, we can prove the following upper bound. The details of the algorithm are in Algorithm \ref{alg:1}.

\begin{algorithm}
    \caption{SVCTs via Denoised Diffusion Smoothing}
    \label{alg:1}
    \begin{algorithmic}[1]
        \State {\bfseries Input:} $X$; A standard deviation $\sigma > 0$. 
        \State $t^{*}$, find $t$ s.t. $\frac{1-\beta_t}{\beta_t} = \sigma^2$. 
        \State $X_{t^{*}} = \sqrt{\beta_{t^{*}}} (\tilde{X} + \mathcal{N}(0, \sigma^2 \textbf{I})) $.
        \State $\hat{X} = \text{denoise}(X_{t^{*}}; t^{*})$.
        \State $w = f_c(\hat{X})$, where $f_c$ is in (\ref{eq:100}). 
        \State {\bfseries Return:} Concept feature vector $w$.
    \end{algorithmic}
\end{algorithm} 

\begin{theorem}\label{thm:5.1}
Consider the function  $\tilde{w}(X)=f_c(T(X+S))$,  where $f_c$ as the function in (\ref{eq:100}), $T$ as the denoised diffusion model and $S\sim \mathcal{N}(0, \sigma^2 I_{d_{model}\times n})$. Then, it is an
$(R, D_\alpha, \gamma,  \beta, k,  \|\cdot\|_2)$-{stable concept module} for VCTs for any $\alpha> 1$ if for any input image $x$ we have
\begin{align*}
\sigma^2 & \geq \max\{{\alpha R^2} / 2(\frac{\alpha}{\alpha-1}\ln(2k_0(\sum_{i\in \mathcal{S}}\tilde{w}^\alpha_{i^*})^\frac{1}{\alpha} + \\
& (2k_0)^\frac{1}{\alpha}\sum_{i\not\in \mathcal{S}}\tilde{w}_{i^*})-\frac{1}{\alpha-1}\ln (2k_0)),  {\alpha R^2} /2\gamma\}. \nonumber
\end{align*}
\end{theorem} 
\begin{figure*}[th]
    \centering    \includegraphics[width=1\linewidth]{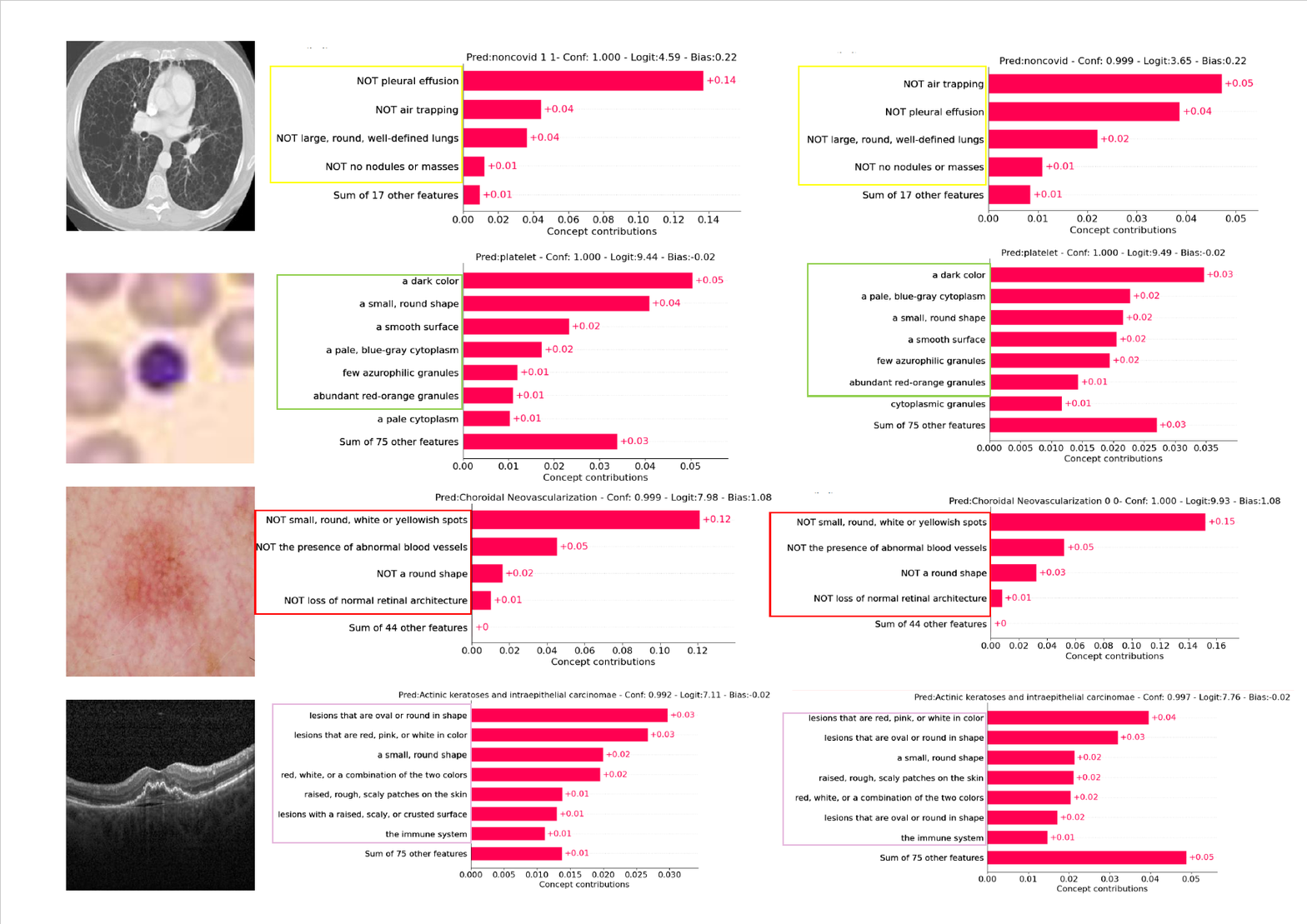}
    \caption{Results of concept visualization. From left to right: one sample from each dataset, concept visualization results before perturbation, and concept visualization results after perturbation. Clear and enlarged pictures are shown in the Appendix L.}
    \label{fig:3}
\end{figure*}

\section{Experiments}
\subsection{Experimental Settings}
\noindent{\bf Datasets.}
We conducted experiments on four medical datasets, including Human Against Machine with 10,015 training images (HAM10000) dataset \cite{Tschandl_2018}, Covid19-CT dataset \cite{zhao2020COVID-CT-Dataset}, BloodMNIST dataset \cite{Yang_2023}, and Optical coherence tomography (OCT) 2,017 dataset \cite{Kermany2018IdentifyingMD}. Details are in Appendix G.

\begin{table}[thbp]
    \centering
    \caption{Results of accuracy for the baselines and SVCT w/w.o perturbation. }
   \resizebox{1\linewidth}{!}{
    \begin{tabular}{lcccc}
    \toprule
        \textbf{Method} & \textbf{HAM10000} & \textbf{Covid19-CT} & \textbf{BloodMNIST} & \textbf{OCT2017} \\
        \bottomrule
        Standard (No interpretability) & $99.13 \%$ & $81.62 \%$ & $97.05 \%$ & $99.70 \%$ \\
        \hline
        Label-Free CBM (LF-CBM) & $93.61 \%$ & $79.75 \%$ & $94.97 \%$ & $97.50 \%$ \\
        Post-hoc CBM (P-CBM) & $97.60 \%$ & $76.26 \%$ & $94.83 \%$ & $98.60 \%$ \\
        \rowcolor{grey!20}
        Vision Concept Transformer (VCT) & $99.00 \%$ & $80.62 \%$ & $96.21 \%$ & $99.10 \%$ \\
        \rowcolor{grey!20}
        \textbf{Stable VCT(SVCT)} & $\textbf{99.05\%}$ & $\textbf{81.37\%}$ & $\textbf{96.96\%}$ & $\textbf{99.50\%}$ \\
        \hline
        $\rho_u=8/255$ - LF-CBM & $90.08 \%$ & $67.98 \%$ & $80.53 \%$ & $91.88 \%$ \\
        $\rho_u=8/255$ - P-CBM & $90.96 \%$ & $70.66 \%$ & $77.55 \%$ & $91.70 \%$ \\
        \rowcolor{grey!20}
        $\rho_u=8/255$ - VCT & $95.80 \%$ & $69.78 \%$ & $89.45 \%$ & $96.80 \%$ \\
        \rowcolor{grey!20}
        $\rho_u=8/255$ - \textbf{SVCT} & $\textbf{97.97\%}$ & $\textbf{74.45\%}$ & $\textbf{94.07\%}$ & $\textbf{98.70\%}$ \\
        \hline
        $\rho_u=10/255$ - LF-CBM & $88.70 \%$ & $65.12 \%$ & $75.63 \%$ & $90.58 \%$ \\
        $\rho_u=10/255$ - P-CBM & $90.21 \%$ & $66.32 \%$ & $74.27 \%$ & $90.10 \%$ \\
        \rowcolor{grey!20}
        $\rho_u=10/255$ - VCT & $95.28 \%$ & $68.85 \%$ & $87.71 \%$ & $96.25 \%$ \\
        \rowcolor{grey!20}
        $\rho_u=10/255$ - \textbf{SVCT} & $\textbf{97.24\%}$ & $\textbf{71.65\%}$ & $\textbf{92.65\%} $ & $\textbf{98.48\%}$ \\
        \bottomrule
    \end{tabular}}
\label{tab:1}   
\end{table}

\noindent{\bf Baselines.}
In this paper, the standard model is ViT \cite{dosovitskiy2020image}, which accomplishes the classification task by extracting image features, but the model itself is not interpretable. The baseline model is label-free CBM \cite{oikarinen2023labelfree}, which uses ViT as the backbone to generate a conceptual bottleneck layer and finally makes predictions through a linear layer.


\begin{table}[thbp]
    \centering
\caption{Results on CFS and CPCS for the baselines and SVCT under various perturbations.}
\resizebox{1\linewidth}{!}{
    \begin{tabular}{lcccccccc}
    \toprule
    \multirow{2}{*}{\textbf{Method}} & \multicolumn{2}{c}{\textbf{HAM10000}} & \multicolumn{2}{c}{\textbf{Covid19-CT}} & \multicolumn{2}{c}{\textbf{BloodMNIST}} & \multicolumn{2}{c}{\textbf{OCT2017}} \\
    \cmidrule{2-9}
    & CFS & CPCS & CFS & CPCS & CFS & CPCS & CFS & CPCS \\
\midrule
$\rho_u=6/255$ - LF-CBM & 0.3335 & 0.9405 & 0.6022 & 0.8117 & 0.5328 & 0.8511 & 0.3798 & 0.9254 \\
$\rho_u=6/255$ - VCT & 0.3361 & 0.9394 & 0.6761 & 0.7650 & 0.5432 & 0.8436 & 0.3625 & 0.9314 \\
\rowcolor{grey!20}
$\rho_u=6/255$ - \textbf{SVCT} & \textbf{0.1354} & \textbf{0.9900} & \textbf{0.5555} & \textbf{0.8359} & \textbf{0.3589} & \textbf{0.9320} & \textbf{0.3257} & \textbf{0.9468} \\
\hline
$\rho_u=8/255$ - LF-CBM & 0.3719 & 0.9256 & 0.6707 & 0.7710 & 0.6280 & 0.7947 & 0.3941 & 0.9196 \\
$\rho_u=8/255$ - VCT & 0.4109 & 0.9098 & 0.8114 & 0.6743 & 0.7162 & 0.7328 & 0.3812 & 0.9240 \\
\rowcolor{grey!20}
$\rho_u=8/255$ - \textbf{SVCT} & \textbf{0.1555} & \textbf{0.9867} & \textbf{0.6446} & \textbf{0.7818} & \textbf{0.4383} & \textbf{0.8977} & \textbf{0.3459} & \textbf{0.9387} \\
\hline
$\rho_u=10/255$ - LF-CBM & 0.4027 & 0.9123 & 0.7224 & 0.7336 & 0.6906 & 0.7545 & 0.4055 & 0.9145 \\
$\rho_u=10/255$ - VCT & 0.4637 & 0.8844 & 0.8943 & 0.6155 & 0.8057 & 0.6670 & 0.3949 & 0.9179 \\
\rowcolor{grey!20}
$\rho_u=10/255$ - \textbf{SVCT} & \textbf{0.1725} & \textbf{0.9836} & \textbf{0.7096} & \textbf{0.7389} & \textbf{0.5058} & \textbf{0.8625} & \textbf{0.3620} & \textbf{0.9321} \\
\bottomrule
\end{tabular}}
    \label{tab:2}
\end{table}

\noindent{\bf Perturbations.}
Perturbation refers to small changes or modifications made to input data. In this paper, we introduce perturbations to input images with different radius $\rho_u$ to assess the stability and robustness of the SVCT model. The range of perturbation radii $\rho_u$ is [$6/255$, $10/255$]. We employ the PGD \cite{madry2017towards} algorithm to craft adversarial examples with a step size of 2/255 and a total of 10 steps. As a default, we set the standard deviation $S = 8/255$ for the Gaussian noise in our method. All results are the average score running 10 times to reduce variance.

\noindent{\bf Evaluation metrics.} 
To demonstrate the utility of our approach, we report the classification accuracy on test data for classification tasks. We evaluate our model's stability using Concept Faithfulness Score (CFS) and Concept Perturbation Cosine Similarity (CPCS). CFS measures the stability of model interpretability between two concept weight vectors using Euclidean distance; we use $c_1$ to represent the concept weight vector without perturbation and $c_2$ to represent the concept weight after the perturbation. Then CFS is defined as $\text{CFS} = \|c_2-c_1\| / \|c_1\|$. CPCS measures the cosine similarity between two concept weight vectors, which is defined as $\text{CPCS} = c_1 \cdot c_2 / \|c_1\|\|c_2\|$. The smaller the value of CFS, the less the conceptual weights change after being perturbed, and the more stable the model interpretability is. The closer the value of CPCS is to 1, the higher the similarity of conceptual weights before and after perturbation and the more stable interpretability of the model. More experimental details are in the Appendix \ref{exper}.

\begin{figure}[th]
    \centering    \includegraphics[width=1\linewidth]{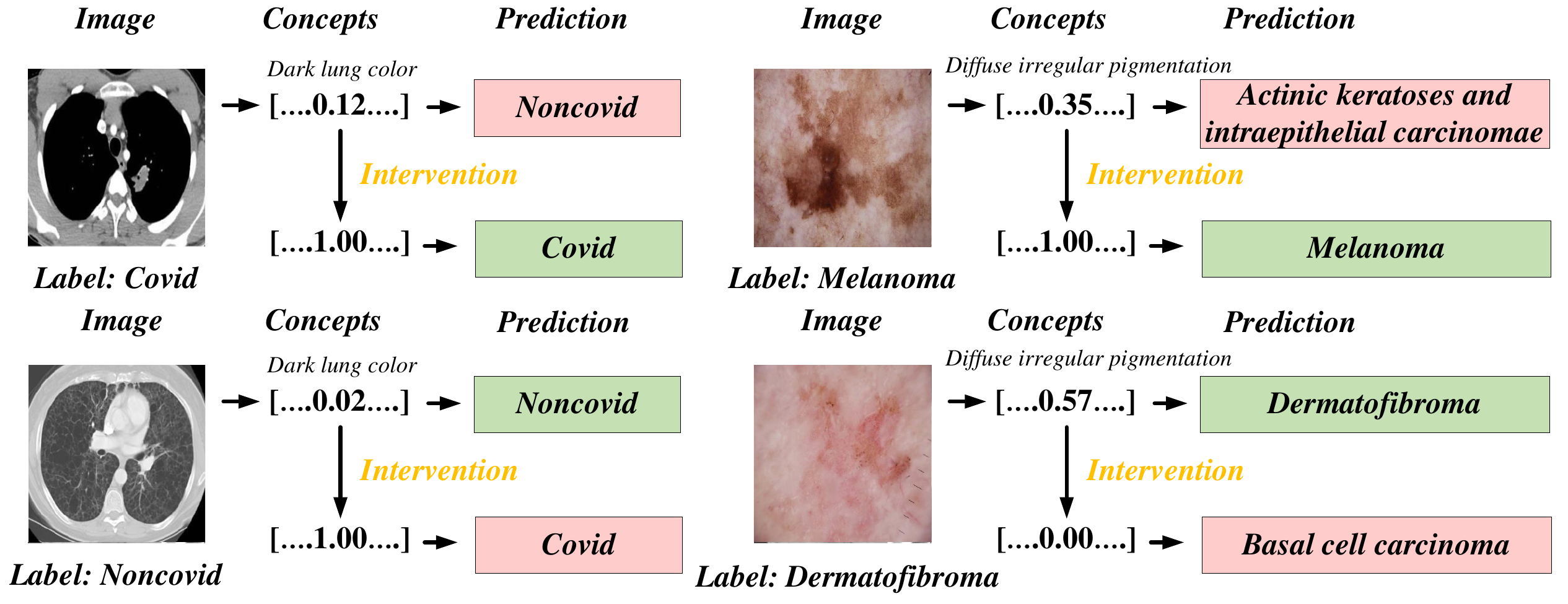}
    \caption{Concept-intervention examples.}
    \label{fig:4}
\end{figure}

\subsection{Utility Evaluation} 
Table \ref{tab:1} presents the accuracy results of our proposed SVCT method and the baseline approach on four datasets with different levels of perturbations. The table clearly shows that our method maintains a consistently high accuracy across all datasets without any noticeable variation or loss. This highlights the robustness of our approach in terms of accuracy preservation. Compared to Label-free CBM, our model can maintain higher accuracy while guaranteeing interpretability. Overall, the results in Table \ref{tab:1} show that our SVCT model successfully combines high accuracy and interpretability and maintains stability over multiple datasets.

\begin{table}[thtbp]
    \centering
\caption{Results on sensitivity and specificity for the baselines and SVCT w/w.o perturbation.}
    \label{sen}
\resizebox{1\linewidth}{!}{
    \begin{tabular}{lcccccccc}
    \toprule
    \multirow{2}{*}{\textbf{Method}} & \multicolumn{2}{c}{\textbf{HAM10000}} & \multicolumn{2}{c}{\textbf{Covid19-CT}} & \multicolumn{2}{c}{\textbf{BloodMNIST}} & \multicolumn{2}{c}{\textbf{OCT2017}} \\
    \cline { 2 - 9 }
    & sensitivity &  specificity & sensitivity &  specificity & sensitivity &  specificity & sensitivity &  specificity \\
\midrule
Label-free CBM                  & 0.8878 & 0.9827 & 0.7984 & \textbf{0.8608} & 0.9407 & 0.9956 & 0.9750 & 0.9960 \\
\rowcolor{grey!20}
\textbf{SVCT}                   & \textbf{0.9899} & \textbf{0.9999} & \textbf{0.8191} & 0.8037 & \textbf{0.9667} & \textbf{0.9958} & \textbf{0.9950} & \textbf{0.9994} \\
\hline
$\rho_u=10/255$ - LF CBM          & 0.6779 & 0.9615 & 0.5794 & \textbf{0.9810} & 0.5880 & \textbf{0.9998} & 0.8380 & 0.9880 \\
\rowcolor{grey!20}
$\rho_u=10/255$ - \textbf{SVCT} & \textbf{0.9180} & \textbf{0.9932} & \textbf{0.7136} & 0.9303 & \textbf{0.8681} & 0.9948 & \textbf{0.9790} & \textbf{0.9923} \\
\bottomrule
\end{tabular}}   
\end{table}

\subsection{Stability Evaluation}
Table \ref{tab:2} illustrates the experimental result for CFS and CPCS, assessing the stability of CBMs across various disturbance radii and comparing it with the baseline models. SVCT demonstrates superior stability concerning conceptual weights, showcasing minimal disparities pre and post-disturbance, signifying notable similarity. The prowess of SVCT in both CFS and CPCS exceeds that of the baseline model. These outcomes imply that SVCT maintains interpretability with robust resistance to perturbation, establishing it as a model with faithful explanations.

In order to represent the experimental results more intuitively, we first visualized the conceptual weight changes before and after the perturbation of each data. The results of these visualizations provide an intuitive explanation of the validity and stability of the SVCT's performance under the perturbation. The results in both Table \ref{tab:2} and Figure \ref{fig:3} amply demonstrate that, compared with the baseline model, the SVCT is a model with superior stability while keeping interpretability to perturbation. These advantages make SVCT valuable in the medical field. Secondly, we also conducted repeated experiments in several conceptual spaces to verify the validity of SVCT. Details can be found in Appendix \ref{more ex}.

\begin{table}[thbp]
    \centering
    \caption{Ablation study of SVCT on DDS module. We assess the efficacy of denoising and smoothing under input perturbations.}
 \resizebox{1\textwidth}{!}{
    \begin{tabular}{lcccccccccc}
    \toprule
    \multirow{2}{*}{\textbf{Method}} & \multicolumn{2}{c}{\textbf{Setting}}&\multicolumn{2}{c}{\textbf{HAM10000}} & \multicolumn{2}{c}{\textbf{Covid19-CT}} & \multicolumn{2}{c}{\textbf{BloodMNIST}} & \multicolumn{2}{c}{\textbf{OCT2017}} \\
    \cmidrule{2-11}
    & Denosing & Smoothing & CFS & CPCS & CFS & CPCS & CFS & CPCS & CFS & CPCS \\
\midrule
\multirow{4}{*}{$\rho_u=6/255$} 
    &            &            & 0.3361 & 0.9394 & 0.6761 & 0.7650 & 0.5432 & 0.8436 & 0.3625 & 0.9314 \\
    &            & \checkmark & 0.3342 & 0.9405 & 0.6490 & 0.7789 & 0.5412 & 0.8462 & 0.3516 & 0.9362 \\
    & \checkmark &            & 0.2689 & 0.9607 & 0.5698 & 0.8221 & 0.3612 & 0.9288 & 0.3367 & 0.9425 \\
    \rowcolor{grey!20}
    & \checkmark & \checkmark & \textbf{0.1354} & \textbf{0.9900} & \textbf{0.5555} & \textbf{0.8359} & \textbf{0.3589} & \textbf{0.9320} & \textbf{0.3257} & \textbf{0.9468} \\
\hline
\multirow{4}{*}{$\rho_u=8/255$} 
    &            &            & 0.4109 & 0.9098 & 0.8114 & 0.6743 & 0.7162 & 0.7328 & 0.3812 & 0.9240 \\
    &            & \checkmark & 0.3716 & 0.9255 & 0.7258 & 0.7288 & 0.6349 & 0.7862 & 0.3724 & 0.9279 \\
    & \checkmark &            & 0.3020 & 0.9503 & 0.6556 & 0.7710 & 0.4560 & 0.8724 & 0.3574 & 0.9343 \\
    \rowcolor{grey!20}
    & \checkmark & \checkmark & \textbf{0.1555} & \textbf{0.9867} & \textbf{0.6446} & \textbf{0.7818} & \textbf{0.4383} & \textbf{0.8977} & \textbf{0.3459} & \textbf{0.9387} \\
\hline
\multirow{4}{*}{$\rho_u=10/255$} 
    &            &            & 0.4637 & 0.8844 & 0.8943 & 0.6155 & 0.8057 & 0.6670 & 0.3949 & 0.9179 \\
    &            & \checkmark & 0.4022 & 0.9119 & 0.7856 & 0.6884 & 0.6940 & 0.7453 & 0.3869 & 0.9217 \\
    & \checkmark &            & 0.3306 & 0.9402 & 0.7157 & 0.7320 & \textbf{0.4988} & 0.8421 & 0.3711 & 0.9283 \\
    \rowcolor{grey!20}
    & \checkmark & \checkmark & \textbf{0.1725} & \textbf{0.9836} & \textbf{0.7096} & \textbf{0.7389} & 0.5058 & \textbf{0.8625} & \textbf{0.3620} & \textbf{0.9321} \\
\bottomrule
    \end{tabular}}
    \label{tab:3}
\end{table}

\subsection{Interpretability Evaluation}
\noindent{\bf Faithfulness and stability.}
SVCT introduces a DDS module while ensuring interpretability, which enables SVCT to provide faithful interpretations, and the results in Table \ref{tab:2} and Figure \ref{fig:3} have shown that the stability performance of SVCT performs even better under input perturbations. Experimental results indicate that SVCT is a faithful model.

\noindent{\bf Test-time intervention.}
We envision that in practical applications, medical experts interacting with the model can intervene to "correct" concept values that the model predicts incorrectly. During the inference process, we initially predict concepts and obtain corresponding concept scores. Subsequently, we intervene by altering concept values and generating output results based on the intervened concepts. In Figure \ref{fig:4}, we present several examples of interventions. In the example, we observed a significant darkening of the lung color, and the model gave an incorrect prediction, which, after our corrections, ended up being correct. When the model predicts correctly, we make the wrong corrections, which likewise causes the model to predict incorrectly. SVCT gives explanations that humans can understand and that humans can modify to achieve co-diagnosis. Besides, our SVCT can also improve its faithfulness in the test-time intervention under perturbations. 

\noindent{\bf Sensitivity and specificity.}
We also conducted sensitivity and specificity experiments on four datasets. Results are shown in Table \ref{sen}. Sensitivity measures the proportion of actual positive cases that are correctly identified by the model and specificity measures the proportion of actual negative cases that are correctly identified by the model. Results show that SVCT consistently outperforms the LF CBM. For the Covid19-CT dataset, while LF CBM has the highest specificity (0.8608), SVCT demonstrates a higher sensitivity (0.8191), suggesting better detection of positive cases. When perturbation ($\rho_u=10/255$), SVCT continues to show robust performance. For example, on the HAM10000 dataset, SVCT maintains high sensitivity (0.9180) and specificity (0.9932). These results demonstrate that SVCT not only performs well under standard conditions but also maintains high accuracy and robustness in the presence of data perturbations, making it a promising method for medical image analysis.

\begin{table}[thbp]
    \centering
\caption{Ablation study of SVCT on DDS module. We assess the efficacy of denoising and smoothing under input perturbations.}
    \resizebox{\linewidth}{!}{
    \begin{tabular}{lcccccccccc}
    \toprule
    \multirow{2}{*}{\textbf{Method}} & \multicolumn{2}{c}{\textbf{Setting}}&\multirow{2}{*}{\textbf{HAM10000}} & \multirow{2}{*}{\textbf{Covid19-CT}} & \multirow{2}{*}{\textbf{BloodMNIST}} & \multirow{2}{*}{\textbf{OCT2017}} \\
    \cmidrule{2-3}
    & Denosing & Smoothing \\
\midrule
\multirow{4}{*}{$\rho_u= 0$} 
    &            &            & $99.00 \%$  & $81.23 \%$ & $96.81 \%$ & $99.40 \%$  \\
    &            & \checkmark & $98.33 \%$  & $80.54 \%$ & $95.88 \%$ & $99.20 \%$  \\
    & \checkmark &            & $98.88 \%$  & $81.09 \%$ & $96.33 \%$ & $99.50 \%$  \\
    \rowcolor{grey!20}
    & \checkmark & \checkmark & $\textbf{99.05\%}$ & $\textbf{81.37\%}$ & $\textbf{96.96\%}$ & $\textbf{99.50\%}$  \\
\hline
\multirow{4}{*}{$\rho_u=10/255$} 
    &            &            & $92.56 \%$ & $68.22 \%$ & $80.59 \%$ & $95.40 \%$  \\
    &            & \checkmark & $92.66 \%$ & $69.10 \%$ & $81.14 \%$ & $97.00 \%$  \\
    & \checkmark &            & $96.11\%$ & $70.03 \%$ & $90.21 \%$ & $98.10 \%$  \\
    \rowcolor{grey!20}
    & \checkmark & \checkmark & $\textbf{97.24\%}$ & $\textbf{71.65\%}$ & $\textbf{92.65\%}$ & $\textbf{98.48\%}$  \\
\bottomrule
    \end{tabular}}
\label{tab:3-acc}
\end{table}



\subsection{Ablation Study}
Results are shown in Table \ref{tab:3} and \ref{tab:3-acc}. The denoising diffusion model and randomized smoothing play an important role in SVCT. When we remove the denoising diffusion model, the performance of the model suffers significantly. While removing the randomized smoothing, the model performance degradation is small. When both modules are removed at the same time, the overall performance of the model decreases more significantly compared to removing a single module. This suggests that these two modules play a key role in maintaining conceptual stability while being able to provide faithful explanations. The ablation results show that without any one of the two modules, the performance of disease diagnosis may suffer. More ablation studies about the effect of feature fusion and DDS are shown in Appendix \ref{ablation study}, indicating that each module in our SVCT is necessary and efficient. The computational cost is shown in Appendix \ref{cost_a}, implying the efficiency of our SVCT.

\section{Conclusion}
In this paper, we propose the Vision Concept Transformer (VCT), and further propose the Stable Vision Concept Transformer (SVCT) framework. In SVCT, we utilize ViT as a backbone, generate the concept layer, and fuse the concept features and image features. SVCT mitigates the information leakage problem caused by CBM and maintains accuracy. Comprehensive experiments show that SVCT can provide stable interpretations despite perturbations to the inputs, with less performance degradation than CBMs and maintaining higher accuracy, indicating SVCT is a more faithful explanation tool.

\section*{Acknowledgements}
This work is supported in part by the funding BAS/1/1689-01-01, URF/1/4663-01-01, REI/1/5232-01-01, REI/1/5332-01-01, and URF/1/5508-01-01 from KAUST, and funding from KAUST - Center of Excellence for Generative AI, under award number 5940.

%
%
%

\bibliographystyle{splncs04}
\bibliography{ref}

\clearpage
\newpage
\appendix

\section{Preliminaries}
\subsection{Vision Transformers}
In this paper, we adopt the notation introduced in \cite{vaswani2017attention} to describe ViTs. ViTs use only the encoder part of the transformer model for feature extraction. For a given input $x$, ViTs divides $x$ into $n$ patches of the same size. Each patch is first converted into a one-dimensional vector, after which it is transformed into a token embedding, denoted as $X_i \in \mathbb{R}^{d_{model}}$. Token embeddings are then fed into the encoder part of the transformer, which accomplishes the token mixing using a multi-head self-attention mechanism, after which the multi-channel features are combined by MLPs.

\noindent{\bf Token mixing.} 
For input $x$, we denote its corresponding token embedding as $X = [X_1, \cdots, X_n] \in \mathbb{R}^{d_{model} \times n}$, and in the self-attention mechanism, query, keys, and values are all inputs themselves. We denote its dimension as $d_k$, so a linear transformation is needed to obtain the query matrix $Q=W_QX \in \mathbb{R}^{d_{k} \times n}$, the keys matrix $K=W_KX \in \mathbb{R}^{d_{k} \times n}$, and the values matrix $V=W_VX \in \mathbb{R}^{d_{k} \times n}$, where $W_Q, W_K, W_V \in \mathbb{R}^{d_{k} \times d_{model}} $ are learnable weight parameters, After that the process of computing token features by the self-attention module can be expressed as:
\begin{equation}\label{eq:1}
Z^{\top} = \text{self-attention}(X) = \text{softmax}(\frac{Q^{\top}K}{\sqrt{d_k}})V^{\top}W_O
\end{equation}
$Z=[z_1, \cdots, z_n]$ is the extracted token feature and $\frac{1}{\sqrt{d_k}}$ is a scaling factor. It is important to note that after obtaining the output of the self-attention module, it is also necessary to transform it into the input dimensions using a linear mapping, where $W_O \in \mathbb{R}^{d_{k} \times d_{model}}$. The output of the self-attention module goes into the MLP after the layer norm to generate the input for the next block.

\noindent{\bf Prediction.} After stacking multiple blocks, the prediction vectors are output in the last layer of ViTs, and the final prediction can be output after one linear layer. It is worth noting that we input X into the self-attention module, and the final result is denoted as $Z(X)$, and we call $Z(X) \in \mathbb{R}^{n}$ the attention feature vector. Finally, we denote the input of the last linear layer of ViTs as $f(X)$. Note that in the VCT framework, $f(X) \in \mathbb{R}^{d_0} $ is also called the backbone feature or ViTs feature, and $d_0$ is the dimension of the backbone feature.

\subsection{Concept Bottleneck Models}
To introduce the original CBMs, we adopt the notations used by \cite{koh2020concept}. We consider a classification task with a concept set denoted as ${c}=\{p_1, \cdots, p_k\}$ and a training dataset represented as $\{ (x_i, y_i, c_i)   \}_{i=1}^{N}$. Here, for $i\in [N]$, $x_i\in \mathbb{R}^d$ represents the feature vector, $y_i\in \mathbb{R}^{d_z}$ denotes the label (with $d_z$ corresponding to the number of classes), and $c_i\in \mathbb{R}^k$ represents the concept vector. In this context, the $j$-th entry of $c_i$ represents the weight of the concept $p_j$. In CBMs, our goal is to learn two representations: one that transforms the input space to the concept space, denoted as $g: \mathbb{R}^d \to \mathbb{R}^k$, and another that maps the concept space to the prediction space, denoted as $f: \mathbb{R}^k \to \mathbb{R}^{d_z}$. For any input $x$, we aim to ensure that its predicted concept vector $\hat{c}=g(x)$ and prediction $\hat{y}=f(g(x))$ are close to their underlying counterparts, thus capturing the essence of the original CBMs.

\section{Notations}
We present our detailed notations in Table \ref{tab:notions}.

\begin{table}[thbp]
\centering
\caption{Notations.}
\resizebox{1\linewidth}{!}{
\begin{tabular}{clcl}
\toprule
\textbf{Notation} & \textbf{Remark} & \textbf{Notation} & \textbf{Remark}  \\
\midrule
$x$ & Input image & $n$ &  \# of patches\\
$X$ & Token embeddings &  $Q, K, V$ & Query,keys,values matrix\\
$W_Q,W_K,W_V,W_O$ & Linear mapping weights & $Z$ & Token feature \\
$Z(X)$& Attention feature vector & $f(X)$ & Backbone feature \\
$\mathcal{C}$ & Concept set & $\mathcal{D}$ & Training dataset\\
$\mathcal{T}$ & Token embedding of $\mathcal{D}$ & $A$ & Activation matrix \\
$E_I$ & CLIP image encoder & $E_T$ & CLIP text encoder  \\
$M$ & \# of concepts & $N$ & \# of data  \\
$f_c(X)$ & Concept feature & $f_m(X)$ & Hybrid features \\
$W_F$ & Weights of final predictor & $V_k$ & Top-k ratio\\ 
$g$ &  Concept module & $R$ & Stable radius \\
$\bar{y}(X)$ & Prediction distribution based on $g(X)$ & $D$ & Rényi divergence\\
$\gamma$ & Similarity coefficient, & $\beta$ & Stability coefficient\\
$\|\cdot\|$ & $\ell_2$-norm or $\ell_{\infty}$-norm& $T$ & Denoised diffusion method\\
$F$ & Concept module for VCT  & $\tilde{w}(x)$ & Stable concept module\\
$\mathcal{N}\left(0, \sigma^2 I\right)$ & Randomized Gaussian noise & $X_{\mathrm{rs}}$ & Noise model of randomize smoothing\\
$X_t$ & Noise model of diffusion models & $t$ & Time step of diffusion model\\
$\beta_t$ & Constant related to time step $t$ & $X_{t^*}$   & Noise model of time step $t^*$ \\
$\hat{X}$ & Denoised sample & $\rho_u$ & Radius of perturbations\\
\bottomrule
\end{tabular}}
\label{tab:notions}
\end{table}

\section{Omitted Proofs}
\label{proof}
We first give the definition of the $\alpha$-R\'{e}nyi divergence. Then, if the prediction distribution is robust under $\alpha$-R\'{e}nyi divergence, then the prediction will be robust under input perturbations \cite{li2019certified}.
\begin{definition}
Given two probability distributions $P$ and $Q$, and  $\alpha\in (1,\infty)$, the $\alpha$-R\'{e}nyi divergence is defined as 
\begin{equation*}
    D_\alpha(P||Q) = \frac{1}{\alpha-1} \log \mathbb{E}_{X\sim Q}(\frac{P(X)}{Q(X)})^\alpha.
\end{equation*}
\end{definition}

\subsection{Proof of Theorem \ref{thm:5.1}}
\begin{proof}
Firstly, we know that the $\alpha$-R\'{e}nyi divergence between two Gaussian distributions $\mathcal{N}(0, \sigma^2 I_d)$ and $\mathcal{N}(\mu, \sigma^2 I_d)$  is bounded by $\frac{\alpha\|\mu\|_2^2}{2\sigma^2}$. Thus, by the post-processing property of R\'{e}nyi divergence, we have 
\begin{equation*}
\begin{aligned}
     & D_\alpha (\tilde{w}(X), \tilde{w}(X'))=D_\alpha (f_c(T(X+S)), f_c(T(X'))) 
     \leq D_\alpha(X+S, X'+S) \\
     & \leq \frac{\alpha\|X-X'\|_F^2}{2\sigma^2}\leq \frac{\alpha R^2}{2\sigma^2}. 
\end{aligned}
\end{equation*}
Thus, when $\frac{\alpha R^2}{2\sigma^2}\leq \gamma$ it satisfies the utility robustness. 

Second, we show it satisfies the prediction robustness. We first recall the following lemma which shows a lower bound between the R\'{e}nyi divergence of two discrete distributions:

\begin{lemma}[R\'{e}nyi Divergence Lemma \cite{li2019certified}]\label{lemma:renyi}
Let $P=(p_1, p_2, ..., p_k)$ and $Q=(q_1, q_2, ..., q_k)$ be two multinomial distributions. If the indices of the largest probabilities \textbf{do not} match on $P$ and $Q$, then the R\'{e}nyi divergence between $P$ and $Q$, {\em i.e.,} $D_\alpha(P||Q)$\footnote{For $\alpha\in (1,\infty)$, $D_\alpha(P||Q)$ is defined as $D_\alpha(P||Q) = \frac{1}{\alpha-1} \log \mathbb{E}_{X\sim Q}(\frac{P(X)}{Q(X)})^\alpha$.}, satisfies
\begin{align}
        D_\alpha(P||Q) 
       \geq -\log(1 - p_{(1)} - p_{(2)} + 2(\frac{1}{2}(p_{(1)}^{1-\alpha} + p_{(2)}^{1-\alpha}))^{\frac{1}{1-\alpha}}).
    \nonumber 
\end{align}
where $p_{(1)}$ and $p_{(2)}$ refer to the largest and the second largest probabilities in $\{p_i\}$, respectively.
\end{lemma}
By Lemma \ref{lemma:renyi} we can see that as long as $  D_\alpha(\tilde{w}(X), \tilde{w}(X')) \leq -\log(1 - p_{(1)} - p_{(2)} + 2(\frac{1}{2}(p_{(1)}^{1-\alpha} + p_{(2)}^{1-\alpha}))^{\frac{1}{1-\alpha}})$ we must have  the prediction robustness. Thus, if $\frac{\alpha R^2}{2\sigma^2}\leq -\log(1 - p_{(1)} - p_{(2)} + 2(\frac{1}{2}(p_{(1)}^{1-\alpha} + p_{(2)}^{1-\alpha}))^{\frac{1}{1-\alpha}})$  we have the condition. 

Finally, we prove the Top-$k$ robustness. Motivated by \cite{liu2021certifiably,hu2023improving}, we proof the following lemma first 
\begin{lemma}\label{lemma:4.3}
Consider the set of all vectors with unit $\ell_1$-norm in $\mathbb{R}^T$, $\mathcal{Q}$. Then we have 
\begin{equation*}
\begin{aligned}
      \min_{q\in \mathcal{Q}, V_k(\hat{w}, q)\geq \beta} 
      D_\alpha(\hat{w}, q) =\frac{\alpha}{\alpha-1}\ln(2k_0(\sum_{i\in \mathcal{S}}\tilde{w}^\alpha_i)^\frac{1}{\alpha} 
       +(2k_0)^\frac{1}{\alpha}\sum_{i\not\in \mathcal{S}}\tilde{w}_i)-\frac{1}{\alpha-1}\ln (2k_0), 
\end{aligned}
\end{equation*}
where $D_\alpha(\hat{w}, q)$ is the $\alpha$-divergence of the distributions whose probability vectors are $\hat{w}$ and $q$. 
\end{lemma}
Now we get back to the proof, we know that $D_\alpha(X+S, X'+S)\leq \frac{\alpha R^2}{2\sigma^2}$. And $D_\alpha(f_c(T(X+S)), T(f_c((X'+S)))\leq D_\alpha(X+S, X'+S)$. Thus, if $\frac{\alpha R^2}{2\sigma^2}\leq \frac{\alpha}{\alpha-1}\ln(2k_0(\sum_{i\in \mathcal{S}}\tilde{w}^\alpha_i)^\frac{1}{\alpha}+(2k_0)^\frac{1}{\alpha}\sum_{i\not\in \mathcal{S}}\tilde{w}_i)-\frac{1}{\alpha-1}\ln (2k_0)$, we must have $V_k(\tilde{w}(X), \tilde{w}(X'))\geq \beta$. 
\end{proof}

\begin{proof}[Proof of Lemma \ref{lemma:4.3}]
We denote $m^T=(m_1, m_2, \cdots, m_T)$ and $q^T=(q_1, \cdots, q_T)$. W.l.o.g we assume that $m_1\geq \cdots\geq m_T$.  Then,  to reach the minimum of R\'{e}nyi
divergence we show that the minimizer $q$ must satisfies $q_1\geq \cdots \geq q_{k-k_0-1}\geq q_{k-k_0}=\cdots=q_{k+k_0+1}\geq q_{k+k_0+2}\geq q_T$.  We need the following statements for the proof. 
\begin{lemma} \label{lemma:4.4}
We have the following statements: 
\begin{enumerate}
    \item 
To reach the minimum, there are exactly $k_0$ different components in the top-$k$ of $\tilde{w}$ and $q$. 
\item To reach the minimum, $q_{k-k_0+1}, \cdots, q_k$ are not in the top-$k$ of $q$. 
\item  To reach the minimum, $q_{k+1}, \cdots, q_{k+k_0}$  must appear in the top-$k$
of $q$. 
\item \cite{li2019certified} To reach the minimum, we must have $q_i\geq q_j$ for all $i\leq j$. 
\end{enumerate}
\end{lemma}
Thus, based on Lemma \ref{lemma:4.4}, we only need to solve the following optimization problem to find a minimizer $q$: 
\begin{align*}
   & \min_{q_1, \cdots, q_T}=\sum_{i=1}^T q_i (\frac{\tilde{w}_i }{q_i})^\alpha \\ 
   &\textbf{s.t. } \sum_{i=1}^T q_i=1 \\
   & \textbf{s.t. } q_i\leq q_j, i\geq j \\
   & \textbf{s.t. } q_i\geq 0\\
   & \textbf{s.t. } q_i-q_j=0, \forall i, j\in \mathcal{S}=\{k-k_0+1, \cdots, k+k_0\} 
\end{align*}
Solve the above optimization by using the Lagrangian method, and we can get 
\begin{equation}
    q_i=\frac{s}{2k_0 s+(2k_0)^\frac{1}{\alpha}\sum_{i\not\in \mathcal{S}}\tilde{w}_i}, \forall i\in \mathcal{S}, 
\end{equation}
\begin{equation}
     q_i=\frac{(2k_0)^\frac{1}{\alpha}\tilde{w}_i}{2k_0 s+(2k_0)^\frac{1}{\alpha}\sum_{i\not\in \mathcal{S}}\tilde{w}_i}, \forall i\not\in \mathcal{S}
\end{equation}
where $s=(\sum_{i\in \mathcal{S}}\tilde{w}^\alpha_i)^\frac{1}{\alpha}$. We can get in this case $D_\alpha(\tilde{w}, q)=\frac{\alpha}{\alpha-1}\ln(2k_0s+(2k_0)^\frac{1}{\alpha}\sum_{i\not\in \mathcal{S}}\tilde{w}_i)-\frac{1}{\alpha-1}\ln (2k_0) $. 

\end{proof}

\begin{proof}[Proof of Lemma \ref{lemma:4.4}]
We first proof the first item: 

Assume that $i_1, \cdots, i_{k_0+j}$ are the $j$ components in the top-k of $\tilde{w}$ but not 
in the top-k of $q$, and  $i'_1, \cdots, i'_{k_0+j}$ 
are the components  in the top-k of q but
not in the top-k of  $\tilde{w}$. Consider we have another vector $q^1$ with
the same value with $q$ while replace $q_{i_{k_0+j}}$ with $q_{i'_{k_0+j}}$. Thus we have 
\begin{align*}
    &e^{(\alpha-1) D_\alpha(\tilde{w}, q^1)}-  e^{(\alpha-1)D_\alpha(\tilde{w}, q)} \\
    & = (\frac{\tilde{w}^\alpha_{i_{k_0+j}} }{q^{\alpha-1
    }_{i'_{k_0+j}}}+\frac{\tilde{w}^\alpha_{i'_{k_0+j}} }{q^{\alpha-1
    }_{i_{k_0+j}}})-(\frac{\tilde{w}^\alpha_{i_{k_0+j}} }{q^{\alpha-1
    }_{i_{k_0+j}}}+\frac{\tilde{w}^\alpha_{i'_{k_0+j}} }{q^{\alpha-1
    }_{i'_{k_0+j}}})\\
    &=(\tilde{w}^\alpha_{i_{k_0+j}} -\tilde{w}^\alpha_{i'_{k_0+j}})(\frac{1}{q^{\alpha-1
    }_{i'_{k_0+j}}}-\frac{1}{q^{\alpha-1
    }_{i_{k_0+j}}})<0,
\end{align*}
since $\tilde{w}_{i_{k_0+j}}\geq \tilde{w}_{i'_{k_0+j}}$ and $q_{i'_{k_0+j}}\geq q _{i_{k_0+j}}$.  Thus, we know reducing the number of misplacement in
top-k can reduce the value $D_\alpha(\tilde{w}, q)$ which contradict to $q$ achieves the minimal. Thus we must have $j=0$. 

We then proof the second statement. 

Assume that $i_1, \cdots, i_{k_0}$ are the $k_0$ components in the top-k of $\tilde{w}$ but not 
in the top-k of $q$, and  $i'_1, \cdots, i'_{k_0}$ 
are the components  in the top-k of q but
not in the top-k of  $\tilde{w}$.  Consider we have another unit $\ell_1$-norm vector $q^2$  with the
same value with $q$ while $q_{i_j}$ is replaced by $q_{j'}$ where $\tilde{w}_{j'}\geq \tilde{w}_{i_j}$ and $j'$ is in the top-k component of $q$ (there must exists such index $j'$). Now we can see that $q^2_{j'}$ is no longer a top-k component of $q^2$ and $q^2_{i_j}$ is a top-k component.  Thus we have 
\begin{align*}
   &e^{(\alpha-1) D_\alpha(\tilde{w}, q^2)}-  e^{(\alpha-1) D_\alpha(\tilde{w}, q)} \\
   & = (\frac{\tilde{w}^\alpha_{i_{j}} }{q^{\alpha-1
    }_{j'}}+\frac{\tilde{w}^\alpha_{j'} }{q^{\alpha-1
    }_{i_{j}}})-(\frac{\tilde{w}^\alpha_{i_{j}} }{q^{\alpha-1
    }_{i_{j}}}+\frac{\tilde{w}^\alpha_{j'} }{q^{\alpha-1
    }_{j'}})\\
    &=(\tilde{w}^\alpha_{i_{j}} -\tilde{w}^\alpha_{j'})(\frac{1}{q^{\alpha-1
    }_{j'}}-\frac{1}{q^{\alpha-1
    }_{i_{j}}})\geq 0. 
\end{align*}
Now we back to the proof of the statement. We first proof $q_k$ is not in the top-k of $q$. If not, that is $k\not\in \{i_1, \cdots, i_{k_0}\}$ and all $i_j<k$. Then we can always find an $i_j<k$ such that $\tilde{w}_k\leq \tilde{w}_{i_j}$, we can find a vector $\tilde{q}$ by replacing $q_{i_j}$ with $q_k$. And we can see that $   D_\alpha(\tilde{w}, \tilde{q})-  D_\alpha(\tilde{w}, q)\leq 0$, which contradict to that $q$ is the minimizer. 

We then proof $q_{k-1}$ is not in the top-k of $q$. If not we can construct $\tilde{q}$ by replacing $q_{k}$ with $q_{k-1}$. Since $q_k$ is not in top-k and $\tilde{w}_{k}\leq \tilde{w}_{k-1}$. By the previous statement we have $   D_\alpha(\tilde{w}, \tilde{q})-  D_\alpha(\tilde{w}, q)\leq 0$, which contradict to that $q$ is the minimizer. Thus, $q_{k-1}$ is not in the top-k of $q$. We can thus use induction to proof statement 2. 

Finally we proof statement 3. We can easily show that $q_{i}\geq q_{k+1}$ for $i\leq k$, and $q_{i}\leq q_{k+1}$ for $i\geq k+2$. Thus, $q_1, \cdots, q_k$ are greater than the left entries. Since by Statement 2 we have $q_{k-k_0}, \cdots q_k$ are not top $k$. Thus we must have $q_{k+1}, \cdots q_{k+k_0} $ must be top-k of $q$. 

\end{proof}

\section{Label-free CBM}
\label{concept set}
\subsection{Concept Set Generation}
In the original CBM paper \cite{koh2020concept}, the generation of concepts set was decided by experts within the application domain, which required a great deal of expertise. Our goal was to enable the entire process to be automated, so we used GPT-3 \cite{brown2020language} via the OpenAI API to generate concept sets. Since GPT-3 is stocked with a great deal of expertise in the medical domain when it is correctly questioned, it is possible to efficiently output important features for recognizing a certain category. In this paper, we ask GPT-3 the following questions:
\begin{itemize}
  \item 
  List the most important features for recognizing something as a \{dataset-image-class\} of \{class\}:    
  \item
  List the things most commonly seen around a \{class\}:
  \item
  Give superclasses for the word \{class\}:
\end{itemize}
Note that \{dataset-image-class\} refers to the medical image type in the corresponding dataset, e.g., CT, etc., and \{class\} corresponds to the category in the image classification task. For GPT-3 to perform well on the above prompt, we provide two examples of the desired outputs for few-shot adaptation. Note that those two examples can be shared across all datasets, so no additional user input is needed to generate a concept set for a new dataset. To reduce variance, we run each prompt three times and combine the results. Combining the concepts received from different classes and prompts gives us a large, somewhat noisy set of initial concepts. Specific examples can be found in Appendix \ref{example}.

\subsection{Concept Set Filtering}
After obtaining the initial set of concepts, we need to use several filters to filter the initial concepts set to improve the quality of the concepts set. See more details are in \cite{oikarinen2023labelfree}. The process of filtering consists of the following main aspects:
\begin{itemize}
  \item [(1)] 
  (\textit{Concept length}) Because of the relatively long description of features in medical imaging, we delete any concept longer than 40 characters in length to keep the concept simple and avoid redundant complexity.   
  \item [(2)]
  (\textit{Remove concepts too similar to classes}) We remove all concepts that are too similar to the names of target classes. We measure this with cosine similarity in a text embedding space. In particular, we use an ensemble of similarities in the CLIP ViT-B/16 text encoder as well as the all-mpnet-base-v2 sentence encoder space, so our measure can be seen as a combination of visual and textual similarity. We deleted concepts with similarity $>$ 0.85 for all datasets to any target class.
  \item [(3)]
  (\textit{Remove concepts too similar to each other}) We use the same embedding space as above and remove any concept that has another concept with $>$ 0.9 cosine similarity already in the concept set.
  \item [(4)]
  (\textit{Remove concepts not present in training data}) To make sure our concept layer accurately presents its target concepts, we remove any concepts that do not activate CLIP highly. This cut-off is dataset-specific, and we delete all concepts with average top-5 activation below the cut-off.
  \item [(5)]
  (\textit{Remove concepts we cannot project accurately}) Remove neurons that are not interpretable from the CBL.
\end{itemize}

\subsection{Learning the Concept Bottleneck Layer.}
After the first step, we obtain the set of human-understandable concepts, next we need to learn how to project from the feature space of the backbone network to an interpretable feature space that corresponds to the set of interpretable concepts in the axial direction. We use a way of learning the projection weights $W_c$ without any labeled concept data by utilizing CLIP-Dissect \cite{oikarinen2023clipdissect}. To start with, we need a set of target concepts that the bottleneck layer is expected to represent as $\mathcal{C}=\left\{c_1, \ldots, c_M\right\}$, as well as a training dataset (e.g., images) $\mathcal{D}=\left\{x_1, \ldots, x_N\right\}$ of the original task, and its corresponding token embedding is denoted as $\mathcal{T}=\left\{X^{(1)},\ldots, X^{(N)}\right\}$, where $N$ is the number of samples. Next, we calculate and save the CLIP concept activation matrix $A$ where $A_{i, j}=E_I\left(x_i\right) \cdot E_T\left(c_j\right)$ and $E_I$ and $E_T$ are the CLIP image and text encoders respectively. $W_c$ is initialized as a random $M \times d_0$ matrix where $d_0$ is the dimensionality of backbone features $f(X)$. We define $f_c(X)=W_c f(X)$, where $f_c\left(X^{(i)}\right) \in \mathbb{R}^M$. We use $e$ to denote a neuron of interest in the projection layer, and its activation pattern is denoted as $q_e$ where $q_e=$ $\left[f_{c, e}\left(X^{(1)}\right), \ldots, f_{c, e}\left(X^{(N)}\right)\right]^{\top}$, with $q_e \in \mathbb{R}^N$ and $f_{c, e}(X)=\left[f_c(X)\right]_e$. Our optimization goal is to minimize the objective $L$ over $W_c$ as follows:
\begin{equation}\label{eq:4}
L\left(W_{c}\right)=\sum_{i=1}^{M}-\operatorname{sim}\left(c_{i}, q_{i}\right):=\sum_{i=1}^{M}-\frac{\bar{q}_{i}^{3} \cdot \bar{A}_{:, i}{ }^{3}}{\left\|\bar{q}_{i}{ }^{3}\right\|_{2}\left\|\bar{A}_{:, i}{ }^{3}\right\|_{2}} .\nonumber
\end{equation}
Here $\operatorname{sim}\left(c_{i}, q_{i}\right)$ is a new fully differentiable similarity function that can be applied to CLIP-Dissect, called cos cubed. $\bar{q}$ indicates vector $q$ normalized to have mean 0 and standard deviation 1.

\section{Example of Step 2}
\label{example}
Figure \ref{fig:enter-label} provides examples of our full prompts for GPT-3 and GPT outputs. For
all experiments, we use the text-davinci-002 model available through OpenAI API. We apply various filters to enhance the quality and reduce the size of our concept set. The filters include: removing concepts longer than 40 characters, eliminating concepts that are too similar to target classes using cosine similarity in a text embedding space with a similarity threshold of 0.85, and removing duplicate or synonymous concepts with a cosine similarity threshold $>$ 0.9.

\begin{figure}[th]
    \centering
    \includegraphics[width=0.85\linewidth]{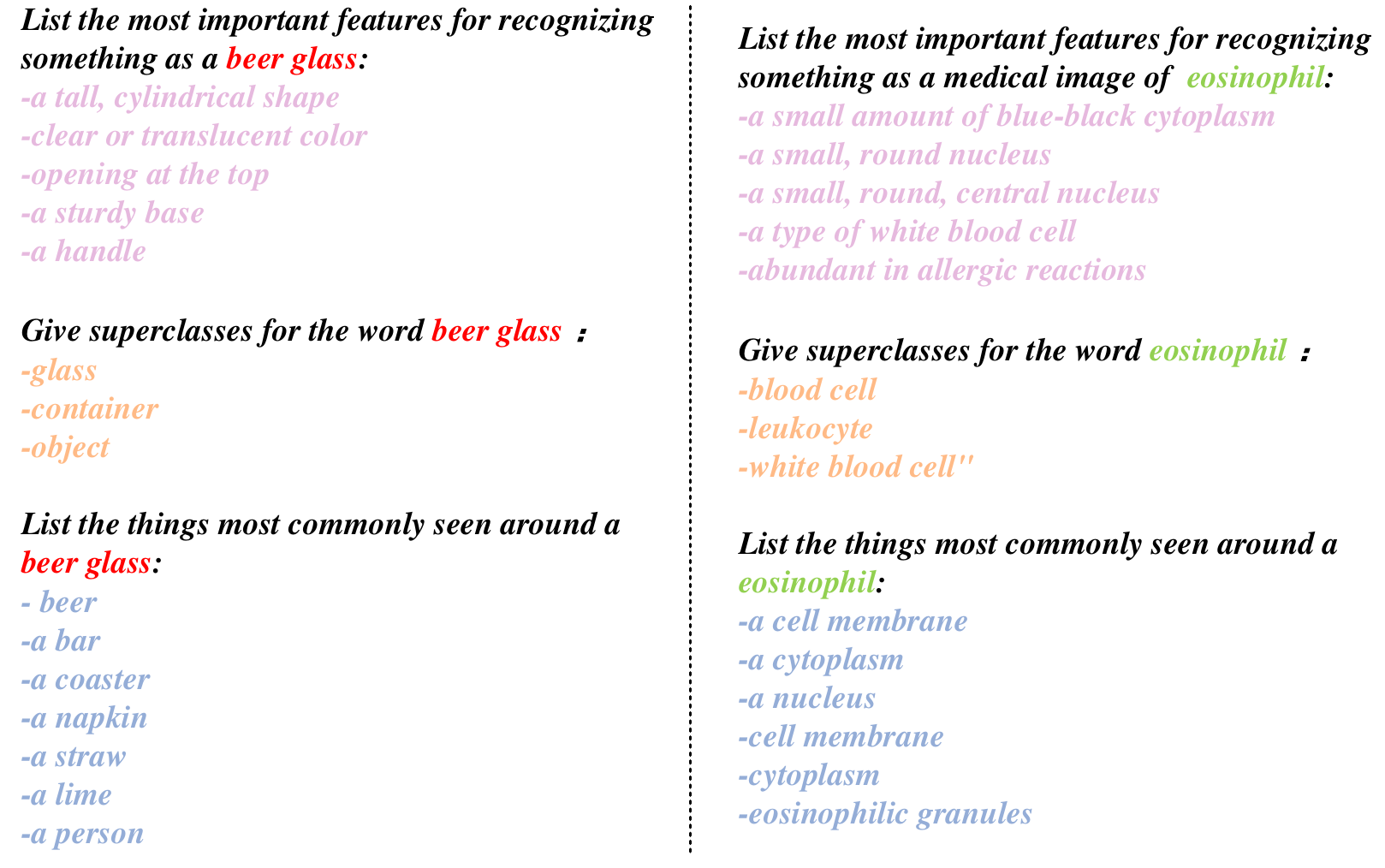}
    \caption{Example of our Step 2.}
    \label{fig:enter-label}
\end{figure}

\section{More Related Work}
\label{related}
\paragraph{Medical Image Classification.} Image and video has attracted much attention in recent years \cite{fang2025your,fang2022multi,fang2023hierarchical,fang2024fewer,fang2023you}, but it is important and complex in the field of medical image analysis. Researchers continue to advance the development of medical image classification techniques by applying different algorithms and methods. Origin medical image classification methods were mainly based on traditional machine learning techniques such as K-classifier, additive regression, bagging, input mapped classifier, decision table, and hand-designed feature extraction methods \cite{Ismael2020MedicalIC}. These methods have achieved some success in the field of medical image classification, but their performance is limited when dealing with complex medical image tasks. With the rise of deep learning, deep neural networks have become a key technology in medical image classification. Deep learning models such as Convolutional Neural Networks (CNNs) have achieved great success in medical image classification with high performance and accuracy. Some of the advanced methods include transfer learning \cite{kim2022transfer}, attention mechanism \cite{li2023deep}, and deep convolutional neural networks \cite{yadav2019deep}. While advanced algorithms have made significant progress in the field of medical image classification, the ensuing problems of black-box nature and instability have become pressing challenges \cite{rudin2019stop}. The complexity of advanced techniques, such as deep learning, leads to opacity in model decisions, making it difficult to explain their behavior in specific contexts. At the same time, the model's sensitivity to input data may lead to inconsistent results in the face of noise or small changes, reducing the model's robustness \cite{szegedy2014intriguing}. There is a need to continuously seek a balance between performance and interpretability.

\paragraph{Robustness for ViTs.}
There is also a substantial body of work on achieving robustness for ViTs, including studies such as \cite{mahmood2021robustness,salman2022certified,aldahdooh2021reveal,naseer2021intriguing,paul2022vision,mao2022towards}. However, these studies exclusively focus on improving the model's robustness in terms of its prediction, without considering the stability of its interpretation (i.e., attention feature vector distribution). While we do employ the randomized smoothing approach commonly used in adversarial machine learning, our primary objective is to maintain the top-$k$ indices unchanged under perturbations. We introduce DDS, which leverages a smoothing-diffusion process to obtain stable VCT while also enhancing prediction performance.

\section{Detailed Experimental Settings}
\label{exper}
\subsection{Datasets}
\paragraph{HAM10000.} Human Against Machine with 10,015 training images (HAM10000) dataset \cite{Tschandl_2018}  is a reliable dataset consisting of 10,015 skin lesion images with high diversity among skin lesion classes. HAM10000 is a seven-level skin lesion classification dataset from a variety of modalities and populations. This dataset includes all significant categories in the pigmented lesion realm, with more than half of the lesion images verified by pathologists and the remaining lesions confirmed by follow-up examination, expert consensus, or in vivo confocal microscopy. Therefore, the HAM10000 dataset was used in this paper for lesion classification. Due to its unbalanced distribution of the number of samples with different labels, new samples are added by random sampling from a small number of classes, so the proportion of samples in each class is 1.

\paragraph{Covid19-CT.} This dataset has a total of 746 lung CT images and provides default divisions for train, val, and test data \cite{zhao2020COVID-CT-Dataset}. In addition, for each new coronavirus-infected CT image, this dataset gives a description of the basic information of the corresponding patient, and this dataset is intended to promote the study of algorithms for the identification of new coronavirus infections in lung CT (2D). For the Covid19-CT dataset, we use a randomized cropping strategy to perform data enhancement.

\paragraph{BloodMNIST.} BloodMNIST dataset is a subset of the MedMNIST benchmark dataset collection \cite{Yang_2023}. This dataset contains images of individual normal cells from individuals who do not have any infections, blood disorders, or tumor diseases. These individuals also did not receive any medications at the time of blood collection. The dataset is categorized into eight types of cells: neutrophils, eosinophils, basophils, lymphocytes, monocytes, immature granulocytes, red blood cells, and platelets. For the BloodMNIST dataset, we need to use mean=0.5 and std=0.5 for normalization.

\paragraph{OCT2017.} Optical coherence tomography (OCT) is an imaging modality capable of viewing the morphology of the retina layer. Therefore, it is the most commonly used in diagnosing and further evaluating macular disease. This paper uses a public OCT dataset named OCT2017 \cite{Kermany2018IdentifyingMD}. The dataset comprised 84,484 OCT images with four retinal disease classes (Normal, CNV, DME, Drusen) divided into three folders (train, evaluation, and test).

\subsection{Backbone}
Vision transformer(ViT) \cite{dosovitskiy2020image} uses self-attention modules. Similar to tokens in the text domain, ViTs divide each image into a sequence of patches (visual tokens) and then feed them into self-attention layers to produce representations of correlations between visual tokens. We use the pre-trained backbones in the timm library for classification. We both leverage the base version with a patch size of 16 and an image size of 224. 

\subsection{Number of Concepts}
In our framework, the number of concepts used in each dataset is related to the number of its categories. For example, the HAM10000 model integrates 79 concepts, the Covid19-CT model utilizes 21, the BloodMNIST model utilizes 82 and the OCT2017 model utilizes 48.

\subsection{Baselines}
\paragraph{Standard.} The standard model functions as an image classification model, extracting image features using the identical backbone as our SVCT model. It then connects a fully connected layer to accomplish the image classification task. In this paper, the standard model is ViT.

\paragraph{Lable-free CBM.} Using a neural network backbone, the Label-free CBM converts the backbone into an interpretable CBM without requiring concept labels, following these four steps – Step 1: Establish the initial concept set and filter out undesired concepts; Step 2: Calculate embeddings from the backbone and the concept matrix on the training dataset; Step 3: Train projection weights $W_c$ to establish a Concept Bottleneck Layer (CBL); Step 4: Train the weights $W_F$ of the sparse final layer for making predictions.

\subsection{Experimental Setup}
Table \ref{tab:4} is presented for a comprehensive overview of our experimental setup, enumerating the crucial parameters employed in our training and evaluation procedures. The selection of these parameter values draws upon prior research and experimental insights, with meticulous adjustments made to ensure optimal performance. It is important to note that these parameters encompass the model architecture and optimizer type and pivotal settings such as learning rate, batch size, number of training iterations, and more. Consultation of Table \ref{tab:4} enables readers to grasp the specific configuration of our experiment and facilitates reproducibility if required.

\begin{table}[thbp]
\centering
\caption{Model parameter configuration.}
\label{tab:4}
\resizebox{1\linewidth}{!}{
\begin{tabular}{lll}
\toprule
\textbf{Argument} & \textbf{Value} & \textbf{Remark} \\
\midrule
batch\_size & 512 & Batch size used when saving model/CLIP activations \\
saga\_batch\_size & 256 & Batch size used when fitting final layer \\
proj\_batch\_size & 5000 & Batch size to use when learning projection layer \\
clip\_cutoff & 0.25 & concepts with smaller top5 clip activation will be deleted \\
proj\_steps & 1000 & how many steps to train the projection layer for \\
interpretability\_cutoff & 0.45 & concepts with smaller similarity to target concept will be deleted \\
lam & 0.0007 & Sparsity regularization parameter, higher-more sparse \\
n\_iters & 1000 & How many iterations to run the final layer solver for \\
$\rho_u$ & [$6/255$, $10/255$] & \\
$S$ &  $8/255$ & \\
trial\_num&  5 & \\
\bottomrule
\end{tabular}}
\end{table}

\section{Additional Ablation Study}
\label{ablation study}
\paragraph{Effect of Feature Fusion and DDS.}
In this paper, we solve the problem of accuracy degradation caused by information leakage by fusing the concept feature and backbone feature. In this part, we conduct ablation experiments for the feature fusion module, and the model without the feature fusion module is label-free CBM. It should be noted that after adding the DDS module to the label-free CBM alone, its performance is basically the same as that of the SVCT in terms of interpretation stability, so we do not show the results of the stability ablation experiments here.
\begin{table}[th]
    \centering
    \caption{Results of ablation study on SVCT. We assess the efficacy of DDS and feature fusion under input perturbation.}
    \resizebox{0.85\textwidth}{!}{
    \begin{tabular}{lcccccccccc}
    \toprule
    \multirow{2}{*}{\textbf{Method}} & \multicolumn{2}{c}{\textbf{Setting}}&\multirow{2}{*}{\textbf{HAM10000}} & \multirow{2}{*}{\textbf{Covid19-CT}} & \multirow{2}{*}{\textbf{BloodMNIST}} & \multirow{2}{*}{\textbf{OCT2017}} \\
    \cmidrule{2-3}
    & Feature Fusion & DDS \\
\midrule
\multirow{4}{*}{$\rho_u= 0$} 
    &            &            & $93.61 \%$ & $79.75 \%$ & $94.97 \%$ & $97.50 \%$  \\
    &            & \checkmark & $94.32 \%$  & $79.88 \%$ & $95.02 \%$ & $97.32 \%$  \\
    & \checkmark &            & $99.00 \%$  & $81.23 \%$ & $96.81 \%$  & $99.40\%$  \\
    \rowcolor{grey!20}
    & \checkmark & \checkmark & $\textbf{99.05\%}$ & $\textbf{81.37\%}$ & $\textbf{96.96\%}$ & $\textbf{99.50\%}$  \\
\hline
\multirow{4}{*}{$\rho_u=10/255$} 
    &            &            & $88.70 \%$ & $65.12 \%$ & $75.63 \%$ & $90.58\%$  \\
    &            & \checkmark & $90.17 \%$ & $67.32 \%$ & $80.43 \%$ & $92.37 \%$  \\
    & \checkmark &            & $92.56 \%$ & $68.22 \%$ & $80.59 \%$ & $95.40 \%$  \\
    \rowcolor{grey!20}
    & \checkmark & \checkmark & $\textbf{97.24\%}$ & $\textbf{71.65\%}$ & $\textbf{92.65\%}$ & $\textbf{98.48\%}$  \\
\bottomrule
    \end{tabular}}
\end{table}

\section{Computational Cost}
\label{cost_a}
In our framework, we use ViTs as the backbone, and in this part of the experiments, we show an example of a model applied to the OCT2017 dataset, where the number of concepts used in the model is 56, and the finalized task is a quadruple categorization, and the dimensions of the model are shown in Table \ref{cost}.
\begin{table}[thbp]
    \centering
    \caption{Results of computational cost.} 
    \label{cost}
    \resizebox{0.7\linewidth}{!}{
    \begin{tabular}{lcccc}
        \toprule
        \textbf{ } & \textbf{ViTS} & \textbf{Label-free CBM} & \textbf{SVCT} \\
        \midrule
        num\_params & 85802728 & 85762568(+40160) & 85845960(+43232)\\
        GFLOPS &  17.56 & 17.56 & 17.56\\
        \bottomrule
    \end{tabular}}
\end{table}

\section{Limitations and Social Impacts}
\label{limitations}
\paragraph{Limitations.} Although our model maintains good accuracy while ensuring interpretability, it still has some limitations. First, SVCT can best be used in collaboration with medical experts as the human evaluation for interpretation quality. Second, our model provides stable explanations in the face of noisy perturbations. We only tested it in the case of Gaussian noise, which is the most common in healthcare settings. Other situations in real healthcare environments still differ from Gaussian noise, which requires further testing. However, our theory proved that Gaussian noise is near-optimal and gave the worst-case of perturbations.

\noindent{\bf Social Impacts.}

\noindent{\bf Positive societal impacts:}

\begin{itemize}
    \item Improved transparency in the medical field. The development of explainable AI models like the Stable Vision Concept Transformer (SVCT) can address the concern of transparency in the medical field. By providing interpretable explanations for the model's decisions, SVCT enables healthcare professionals and patients to understand the reasoning behind medical predictions and diagnoses. This transparency can enhance trust in AI systems and facilitate better collaboration between humans and machines.
    \item Human-understandable high-level concepts. Concept Bottleneck Models (CBMs), including SVCT, aim to generate a conceptual layer that extracts high-level conceptual features from medical data. This can be beneficial in the medical field as it allows healthcare professionals to gain insights into the underlying factors influencing the model's predictions. Understanding these high-level concepts can lead to improved medical knowledge, identification of new patterns, and potential discoveries that can benefit patient care and treatment.
    \item Enhanced decision-making capabilities. SVCT leverages conceptual features and fuses them with image features to enhance decision-making capabilities. By incorporating these conceptual features into the model, SVCT can provide a more comprehensive understanding of medical data and make more informed predictions. This has the potential to improve diagnostic accuracy, treatment planning, and patient outcomes.
    \item Faithful explanations under perturbations. SVCT addresses the limitation of CBMs by consistently providing faithful explanations even when faced with input perturbations. This means that the model's interpretability remains stable and reliable, even in challenging scenarios. In the medical field, where data can be noisy or incomplete, having a model that can provide trustworthy explanations despite perturbations can be crucial for making reliable decisions.
\end{itemize}

\noindent{\bf Negative societal impacts:}

\begin{itemize}
    \item Potential reduction in model performance. The paper mentions that CBMs, including SVCT, can negatively impact model performance. While SVCT aims to maintain accuracy while remaining interpretable, there may still be a trade-off between interpretability and performance. If the conceptual layer or the explainability mechanisms introduced in SVCT significantly affect the model's predictive accuracy, it could limit its usefulness in real-world medical applications. However, this situation is not caused by the framework SVCT, which is the common limitation of all concept-based models.
    \item Limited adoption in the medical field. Despite the benefits of SVCT, the paper acknowledges that the use of CBMs in the medical field is severely limited. This limitation could be due to various factors, such as the complexity of implementing CBMs in clinical settings, the need for extensive validation and regulatory approval, or the preference for more traditional, less interpretable models. If the adoption of SVCT or similar models remains limited, the potential societal impacts, both positive and negative, might not be fully realized. Also, this situation is not caused by the framework SVCT, it exists in all medical-oriented interpretable models.
\end{itemize}

\section{More Experiments}
\label{more ex}
\subsection{Experiments on More Conceptual Spaces}
To demonstrate that SVCT can provide more stable explanations within the same conceptual space, we repeatedly generated different conceptual spaces and replicated the experiments in these spaces. The experimental results are shown in table \ref{tab:5}, \ref{tab:6}, \ref{tab:7}, and \ref{tab:8}. Based on the experimental findings, our SVCT demonstrates greater stability than other baselines when subjected to input perturbation, rendering it a more faithful interpretation. Additionally, our approach showcases minimal accuracy degradation compared to the vanilla CBM.
\begin{table}[thbp]
    \centering
    \caption{Results for both the baselines and SVCT on the accuracy. Experiments are repeated under the \textcolor{red}{new-1} concept space.}
    \label{tab:5}
    \resizebox{1\textwidth}{!}{
    \begin{tabular}{lcccc}
    \toprule
        \textbf{Method} & \textbf{HAM10000} & \textbf{Covid19-CT} & \textbf{BloodMNIST} & \textbf{OCT2017} \\
        \midrule
        Standard (No interpretability) & $99.13 \%$ & $81.62 \%$ & $97.05 \%$ & $99.70 \%$ \\
        \hline
        Label-Free CBM (LF-CBM) & $96.11 \%$ & $76.95 \%$ & $95.53 \%$ & $98.20 \%$ \\
        Post-hoc CBM (P-CBM) & $97.10 \%$ & $74.33 \%$ & $95.22 \%$ & $98.30 \%$ \\
        Vision Concept Transformer (VCT) & $99.05 \%$ & $80.32 \%$ & $96.33 \%$ & $99.00 \%$ \\
        \rowcolor{grey!20}
        \textbf{SVCT} & $\textbf{99.10\%}$ & $\textbf{81.00\%}$ & $\textbf{96.93\%}$ & $\textbf{99.40\%}$ \\
        \hline
        $\rho_u=8/255$ - LF-CBM & $92.51 \%$ & $62.31 \%$ & $86.20 \%$ & $94.30 \%$ \\
        $\rho_u=8/255$ - P-CBM & $90.32 \%$ & $67.55 \%$ & $80.21 \%$ & $91.50 \%$ \\
        $\rho_u=8/255$ - VCT & $95.24 \%$ & $70.13 \%$ & $90.14 \%$ & $95.60 \%$ \\
        \rowcolor{grey!20}
        $\rho_u=8/255$ - \textbf{SVCT} & $\textbf{98.12\%}$ & $\textbf{74.56\%}$ & $\textbf{93.93\%}$ & $\textbf{98.60\%}$ \\
        \hline
        $\rho_u=10/255$ - LF-CBM & $91.24 \%$ & $60.87 \%$ & $82.74 \%$ & $92.50 \%$ \\
        $\rho_u=10/255$ - P-CBM & $88.32 \%$ & $65.87 \%$ & $73.22 \%$ & $90.10 \%$ \\
        $\rho_u=10/255$ - VCT & $94.87 \%$ & $68.44 \%$ & $86.53 \%$ & $93.50 \%$ \\
        \rowcolor{grey!20}
        $\rho_u=10/255$ - \textbf{SVCT} & $\textbf{97.63\%}$ & $\textbf{73.77\%}$ & $\textbf{92.74\%}$ & $\textbf{98.40\%}$ \\
        \bottomrule
    \end{tabular}}
\end{table}
\begin{table}[thbp]
    \centering
    \caption{Results for the baselines and SVCT on
CFS and CPCS under various perturbations. Experiments are repeated under the \textcolor{red}{new-1} concept space.} \label{tab:6}
    \resizebox{1\textwidth}{!}{
    \begin{tabular}{lcccccccc}
    \toprule
    \multirow{2}{*}{\textbf{Method}} & \multicolumn{2}{c}{\textbf{HAM10000}} & \multicolumn{2}{c}{\textbf{Covid19-CT}} & \multicolumn{2}{c}{\textbf{BloodMNIST}} & \multicolumn{2}{c}{\textbf{OCT2017}} \\
    \cmidrule{2-9}
    & CFS & CPCS & CFS & CPCS & CFS & CPCS & CFS & CPCS \\
\midrule
$\rho_u=6/255$ - LF-CBM & 0.3417 & 0.9374 & 0.6566 & 0.7748 & 0.5200 & 0.8567 & 0.3742 & 0.9288 \\
$\rho_u=6/255$ - VCT & 0.3401 & 0.9384 & 0.6882 & 0.7533 & 0.5441 & 0.8432 & 0.3662 & 0.9308 \\
\rowcolor{grey!20}
$\rho_u=6/255$ - \textbf{SVCT} & \textbf{0.2659} & \textbf{0.9617} & \textbf{0.5202} & \textbf{0.8482} & \textbf{0.3519} & \textbf{0.9337} & \textbf{0.3411} & \textbf{0.9432} \\
\hline
$\rho_u=8/255$ - LF-CBM & 0.3783 & 0.9228 & 0.7219 & 0.7322 & 0.6053 & 0.8064 & 0.3946 & 0.9193 \\
$\rho_u=8/255$ - VCT & 0.4144 & 0.9032 & 0.8123 & 0.6735 & 0.7215 & 0.7304 & 0.3823 & 0.9233 \\
\rowcolor{grey!20}
$\rho_u=8/255$ - \textbf{SVCT} & \textbf{0.2973} & \textbf{0.9520} & \textbf{0.5927} & \textbf{0.8048} & \textbf{0.4314} & \textbf{0.8999} & \textbf{0.3619} & \textbf{0.9327} \\
\hline
$\rho_u=10/255$ - LF-CBM & 0.4089 & 0.9091 & 0.7711 & 0.6985 & 0.6611 & 0.7713 & 0.4077 & 0.9123 \\
$\rho_u=10/255$ - VCT & 0.4652 & 0.8821 & 0.9011 & 0.6052 & 0.8122 & 0.6621 & 0.4122 & 0.9118 \\
\rowcolor{grey!20}
$\rho_u=10/255$ - \textbf{SVCT} & \textbf{0.3261} & \textbf{0.9417} & \textbf{0.6525} & \textbf{0.7656} & \textbf{0.4983} & \textbf{0.8658} & \textbf{0.3764} & \textbf{0.9245} \\
\bottomrule
\end{tabular}}
\end{table}
\begin{table}[thbp]
    \centering
    \caption{Results for both the baselines and SVCT on the accuracy. Experiments are repeated under the \textcolor{red}{new-2} concept space.}
    \label{tab:7}
    \resizebox{1\textwidth}{!}{
    \begin{tabular}{lcccc}
    \toprule
        \textbf{Method} & \textbf{HAM10000} & \textbf{Covid19-CT} & \textbf{BloodMNIST} & \textbf{OCT2017} \\
    \midrule
        Standard (No interpretability) & $99.13 \%$ & $81.62 \%$ & $97.05 \%$ & $99.70 \%$ \\
        \hline
        Label-Free CBM (LF-CBM) & $95.56 \%$ & $78.82 \%$ & $94.59 \%$ & $97.60 \%$ \\
         Post-hoc CBM (P-CBM) & $96.20 \%$ & $75.12 \%$ & $93.13 \%$ & $98.50 \%$ \\
        Vision Concept Transformer (VCT) & $98.87 \%$ & $80.02 \%$ & $95.98 \%$ & $99.20 \%$ \\
        \rowcolor{grey!20}
        \textbf{SVCT} & $\textbf{99.05\%}$ & $\textbf{80.37\%}$ & $\textbf{96.81\%}$ & $\textbf{99.50\%}$ \\
        \hline
        $\rho_u=8/255$ - LF-CBM & $90.26 \%$ & $68.02 \%$ & $83.22 \%$ & $93.77 \%$ \\
        $\rho_u=8/255$ - P-CBM & $90.01 \%$ & $67.44 \%$ & $82.21 \%$ & $92.50 \%$ \\
        $\rho_u=8/255$ - VCT & $95.41 \%$ & $69.33 \%$ & $91.12 \%$ & $95.40 \%$ \\
        \rowcolor{grey!20}
        $\rho_u=8/255$ - \textbf{SVCT} & $\textbf{98.02\%}$ & $\textbf{72.69\%}$ & $\textbf{94.15\%}$ & $\textbf{98.67\%}$ \\
        \hline
        $\rho_u=10/255$ - LF-CBM & $89.35 \%$ & $66.11 \%$ & $78.53 \%$ & $92.78 \%$ \\
        $\rho_u=10/255$ - P-CBM & $87.54 \%$ & $64.97 \%$ & $75.22 \%$ & $89.90 \%$ \\
        $\rho_u=10/255$ - VCT & $93.22 \%$ & $66.21 \%$ & $87.32 \%$ & $94.50 \%$ \\
        \rowcolor{grey!20}
        $\rho_u=10/255$ - \textbf{SVCT} & $\textbf{97.60\%}$ & $\textbf{71.59\%}$ & $\textbf{92.76\%}$ & $\textbf{98.54\%}$ \\
    \bottomrule
    \end{tabular}}
\end{table}
\begin{table}[thbp]
    \centering
        \caption{Results for the baselines and SVCT on CFS and CPCS under various perturbations. Experiments are repeated under the \textcolor{red}{new-2} concept space.}
    \label{tab:8}
    \resizebox{1\textwidth}{!}{
    \begin{tabular}{lcccccccc}
\toprule
    \multirow{2}{*}{\textbf{Method}} & \multicolumn{2}{c}{\textbf{HAM10000}} & \multicolumn{2}{c}{\textbf{Covid19-CT}} & \multicolumn{2}{c}{\textbf{BloodMNIST}} & \multicolumn{2}{c}{\textbf{OCT2017}} \\
    \cline { 2 - 9 }
    & CFS & CPCS & CFS & CPCS & CFS & CPCS & CFS & CPCS \\
\midrule
$\rho_u=6/255$ - LF-CBM & 0.3440 & 0.9365 & 0.6556 & 0.7759 & 0.5402 & 0.8427 & 0.3749 & 0.9265 \\
$\rho_u=6/255$ - VCT & 0.3566 & 0.9233 & 0.6931 & 0.7488 & 0.5563 & 0.8344 & 0.3690 & 0.9302 \\
\rowcolor{grey!20}
$\rho_u=6/255$ - \textbf{SVCT} & \textbf{0.2015} & \textbf{0.9602} & \textbf{0.5439} & \textbf{0.8395} & \textbf{0.3542} & \textbf{0.9338} & \textbf{0.3507} & \textbf{0.9369} \\
\hline
$\rho_u=8/255$ - LF-CBM & 0.3829 & 0.9205 & 0.7282 & 0.7241 & 0.6306 & 0.7861 & 0.3964 & 0.9171 \\
$\rho_u=8/255$ - VCT & 0.4188 & 0.9017 & 0.8099 & 0.6751 & 0.7322 & 0.7255 & 0.3951 & 0.9199 \\
\rowcolor{grey!20}
$\rho_u=8/255$ - \textbf{SVCT} & \textbf{0.2335} & \textbf{0.9499} & \textbf{0.6467} & \textbf{0.7796} & \textbf{0.4326} & \textbf{0.9005} & \textbf{0.3722} & \textbf{0.9276} \\
\hline
$\rho_u=10/255$ - LF-CBM & 0.4129 & 0.9072 & 0.7779 & 0.6861 & 0.6889 & 0.7465 & 0.4160 & 0.8932 \\
$\rho_u=10/255$ - VCT & 0.4733 & 0.8754 & 0.9123 & 0.5938 & 0.8213 & 0.6574 & 0.4122 & 0.9018 \\
\rowcolor{grey!20}
$\rho_u=10/255$ - \textbf{SVCT} & \textbf{0.2631} & \textbf{0.9392} & \textbf{0.7105} & \textbf{0.7396} & \textbf{0.4991} & \textbf{0.8666} & \textbf{0.3884} & \textbf{0.9201} \\
\bottomrule
    \end{tabular}}
\end{table}

\section{Presentation}
\label{present}
More presentations are shown in Figure \ref{fig:vis1}, \ref{fig:vis2}, \ref{fig:vis3}, and \ref{fig:vis4}.
\begin{figure}[th]
    \centering    \includegraphics[width=0.85\linewidth]{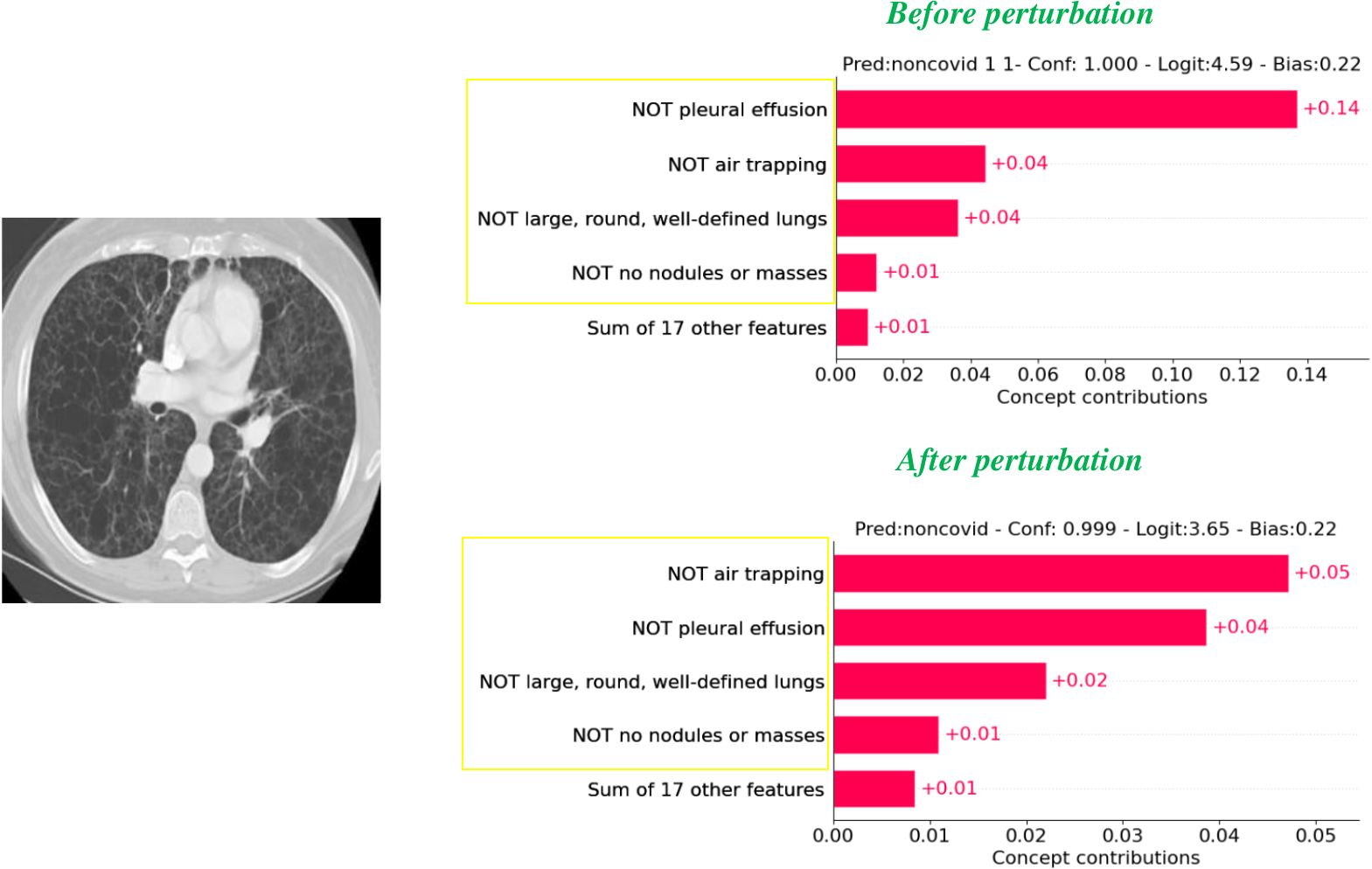}
    \caption{The visualizations for concept weights on one sample from Covid19-CT.}
    \label{fig:vis1}
\end{figure}
\begin{figure}[th]
    \centering
    \includegraphics[width=0.85\linewidth]{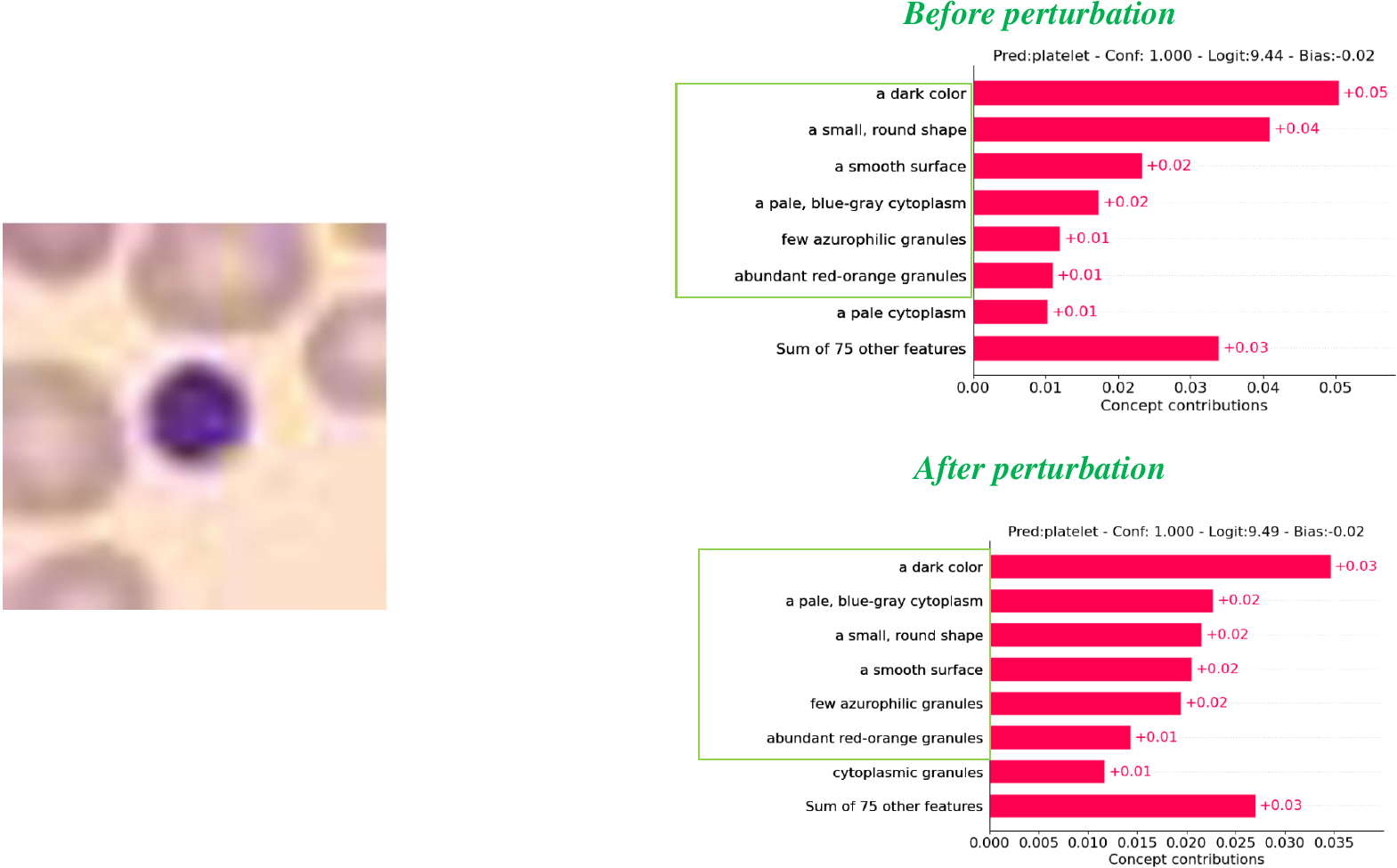}
    \caption{The visualizations for concept weights on one sample from BloodMNIST.}
    \label{fig:vis2}
\end{figure}
\begin{figure}[th]
    \centering
    \includegraphics[width=0.85\linewidth]{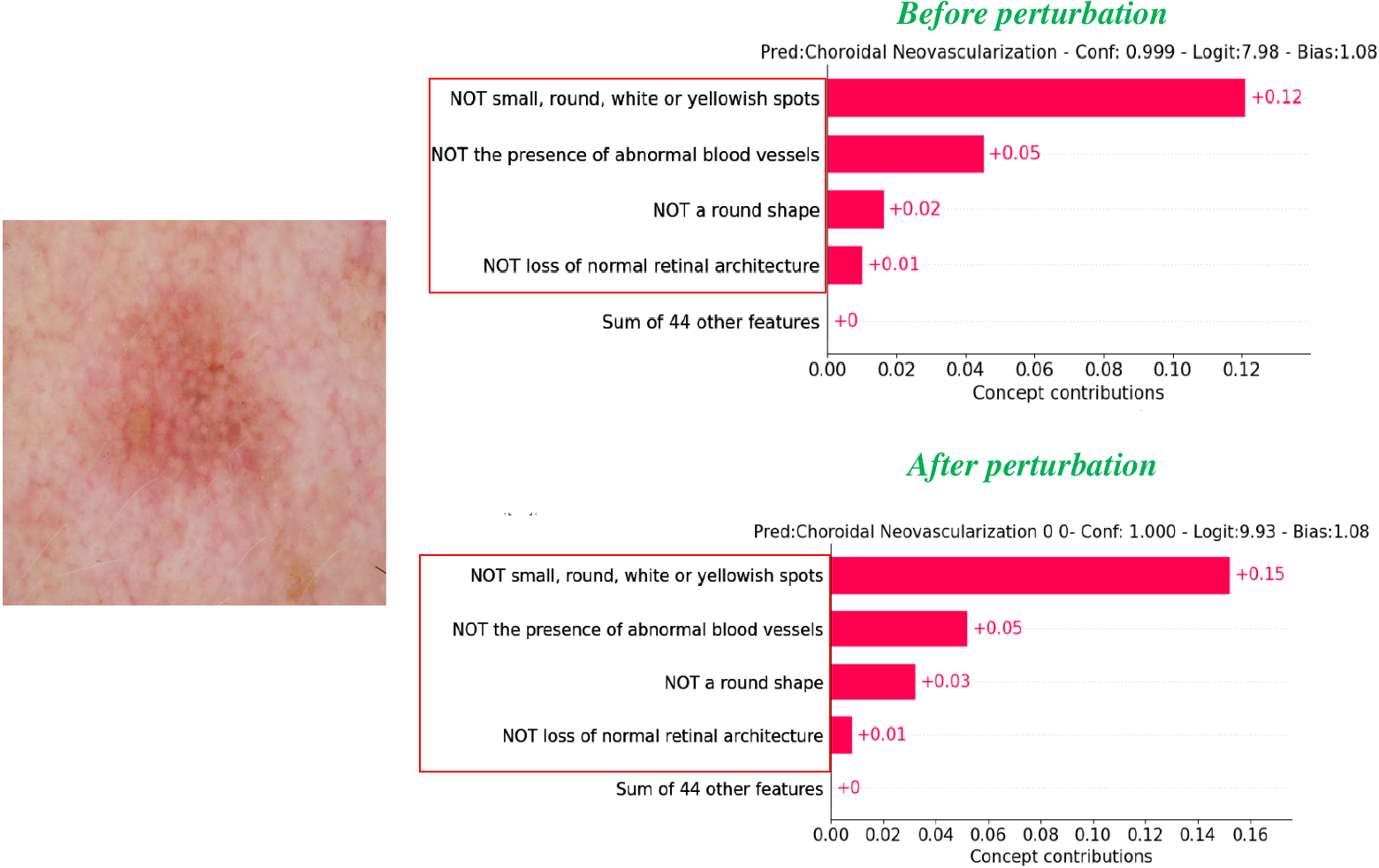}
    \caption{The visualizations for concept weights on one sample from HAM10000.}
    \label{fig:vis3}
\end{figure}
\begin{figure}[th]
    \centering
    \includegraphics[width=0.85\linewidth]{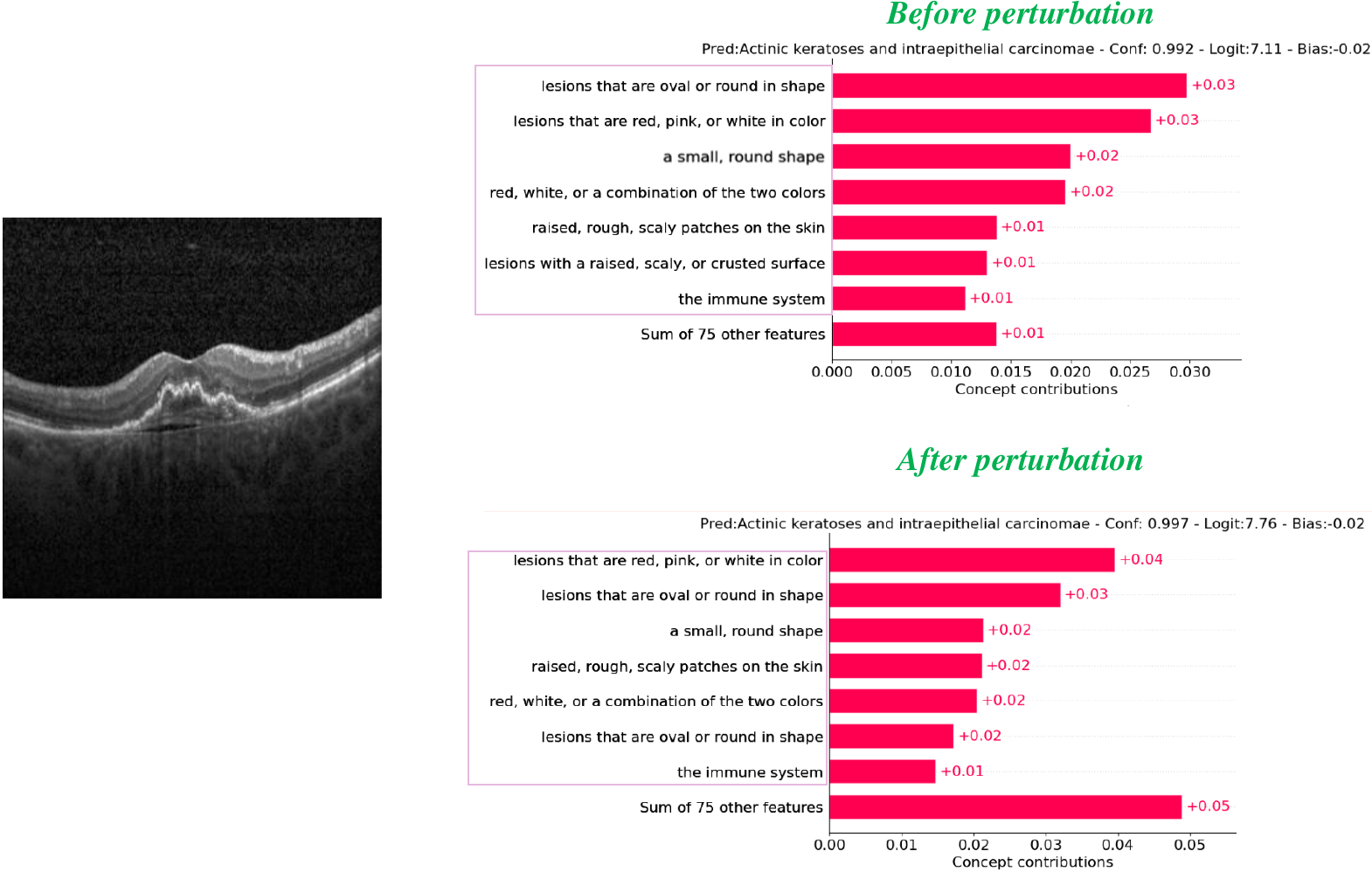}
    \caption{The visualizations for concept weights on one sample from OCT2017.}
    \label{fig:vis4}
\end{figure}

%

\end{document}